\documentclass[11pt,letterpaper]{article}

\usepackage[T1]{fontenc}    %
\usepackage{lmodern}        %
\usepackage{microtype}      %

\usepackage{setspace}
\usepackage{amsmath,amssymb,amsfonts}
\usepackage{amsthm}
\usepackage{fullpage}
\usepackage{graphicx}
\usepackage{mathtools}
\usepackage{xcolor}
\usepackage{authblk}
\usepackage{hyperref}
\usepackage{cleveref}
\usepackage{nicefrac}
\usepackage{microtype}
\usepackage{enumitem}
\usepackage{authblk}
\usepackage{thmtools}
\usepackage{tikz}
\usepackage{natbib}
\usepackage{amartya_ltx}
\usepackage{booktabs}
\usepackage{thm-restate}
\usepackage{multicol}
\usepackage{makecell}
\usepackage{mdframed}
\usepackage{pgfplots}
\usepackage{algorithm}
\usepackage[noend]{algpseudocode}
\usepackage{dsfont}
\usepackage{comment}
\usepackage{wrapfig}
\usepackage{multirow}
\usepackage{mdframed}
\usepackage{float}

\newcommand{\X}{\cX}

\newcommand{\HC}{\mathcal{H}}
\newcommand{\N}{\mathbb{N}}

\newcommand{\HCC}{\overline{\HC}}
\newcommand{\clos}{\operatorname{clos}}

\newcommand{\VS}{\mathrm{VS}^*}
\newcommand{\Trp}{\mathrm{Trap}}
\newcommand{\iid}{\mathrm{i.i.d.}}
\newcommand{\Ldim}{\mathrm{LDim}}

\usepackage[most]{tcolorbox}

\newtcolorbox{defbar}{
  enhanced,
  boxrule=0pt,                %
  frame hidden,               %
  borderline west={3pt}{0pt}{black},  %
  colback=white,              %
  left=8pt,                   %
  right=0pt, top=0pt, bottom=0pt,     %
  before skip=10pt,
  after skip=10pt
}

\usepackage{array}
\usepackage[acronym]{glossaries}

\makeglossaries
\glsdisablehyper %
\newglossaryentry{trap}{
    name={trap region},
    description={Is what the learner can create when predicting with different labels on a point, it is unsure of}
}

\newcolumntype{L}[1]{>{\raggedright\let\newline\\\arraybackslash\hspace{0pt}}m{#1}}
\newcolumntype{C}[1]{>{\centering\let\newline\\\arraybackslash\hspace{0pt}}m{#1}}
\newcolumntype{R}[1]{>{\raggedleft\let\newline\\\arraybackslash\hspace{0pt}}m{#1}}

\usepackage[letterpaper,margin=1in]{geometry}

\title{Learning in an Echo Chamber:\\ Online Learning with Replay Adversary}

\date{}
\author{}
\author{%
  Daniil Dmitriev\textsuperscript{1}\thanks{This work was completed while D. Dmitriev was at ETH Zurich.} \quad
  Harald Eskelund Franck\textsuperscript{2} \quad
  Carolin Heinzler\textsuperscript{2} \quad
  Amartya Sanyal\textsuperscript{2}\thanks{Authors listed in alphabetical order.}\\
  \textsuperscript{1}University of Pennsylvania \quad
  \textsuperscript{2}University of Copenhagen
}

\begin{document}

\maketitle\vspace{-30pt}

\begin{abstract}
As machine learning systems increasingly train on self-annotated data, they risk reinforcing errors and becoming echo chambers of their own beliefs. We model this phenomenon by introducing a learning-theoretic framework:~\emph{Online Learning in the Replay Setting}. In round~\(t\), the learner outputs a hypothesis \(\hat h_t\); the adversary then reveals either the true label \(f^\ast\br{x_t}\) or a \emph{replayed} label \(\hat h_i\br{x_t}\) from an earlier round \(i<t\). A mistake is counted only when the true label is shown, yet classical algorithms such as the SOA or the halving algorithm are easily misled by the replayed errors.\\
We introduce the \emph{Extended Threshold dimension}, \(\extdim{\HC}\), and prove matching upper and lower bounds that make
\(\extdim{\HC}\) the exact measure of learnability in this model. A closure‑based learner makes at most \(\extdim{\HC}\) mistakes against any adaptive adversary, and no algorithm can perform better.  
For stochastic adversaries, we prove a similar bound for every intersection‑closed class. 
The replay setting is provably harder than the classical mistake bound setting: some classes have constant Littlestone dimension but arbitrarily large \(\extdim\HC\). Proper learning exhibits an even sharper separation: a class is properly learnable under replay if and only if it is~(almost) intersection‑closed. Otherwise, every proper learner suffers \(\Om{T}\) errors, whereas our improper algorithm still achieves the \(\extdim{\HC}\) bound. These results give the first tight analysis of learning against replay adversaries, based on new results for closure-type algorithms.
\end{abstract}

\section{Introduction}
\label{sec:intro}

Modern machine learning systems are increasingly trained on the outputs of prior versions of the same or closely related models. A detector-based audit of \emph{Medium} and \emph{Quora} posts found that the share of AI-authored text on those platforms grew from about 2\% in early
2022 to 37-39\% by October 2024~\citep{sun2024we}. 
Even tasks that explicitly require human effort show a similar trend: in a controlled study on Amazon Mechanical Turk,
\(33\text{--}46\%\) of workers were found to paste in the output of a large language model instead of writing their own
summaries~\citep{veselovsky2023artificial}. While the precise numbers remain open to debate, it is clear that today’s learners are often trained not on naturally‑generated data but on the outputs of older machine learning models.

This recursive use of synthetic data is already
routine: vision and language models, as part of their training
pipeline, often pseudo-label new data for self-training~\citep{caron2018deep,YM2020Selflabelling}. Large content-moderation pipelines are label-starved, and it is common
for them to operate in a cascade where they use an older model's
decisions as the ground truth~\citep{hu2024visual,park2023towards}. Edge devices personalise a lightweight
predictor by trusting labels already cached locally, because querying
a remote oracle for every instance is infeasible~\citep{he2019device,liu2023cache}. Formally, in each
case the learner receives a stream \((x_t,y_t)\) where \(y_t\) may be
the prediction of an earlier model as opposed to a natural (possibly
noisy) observation. The naïve learner cannot tell whether a given
observation is new information or a \emph{replay}, so early errors risk being reinforced indefinitely, much like misinformation in echo chambers.

A growing body of work has studied this recursive process, often under
the term \emph{model collapse}~\citep{shumailov2024ai}, but this research has focused almost
exclusively on generative models. These analyses show that
training diffusion models or LLMs on recursively generated synthetic data can cause the tails of the true distribution to vanish~\citep{shumailov2023curse}
and the model's generative quality to degrade over time~\citep{alemohammad2024selfconsuming}.  To the best of our knowledge, there has been no work that views this problem through the lens of
classical learning theory, which is the focus of this paper.

To do so, we identify the core components of the problem. The first
two—that learning happens sequentially over time and that the feedback
received by the learner is uncertain—have been widely studied for
decades in the theory of classical online learning in the presence of
noise~\citep{ben2009agnostic}. The important distinction lies in the third component: the
noise is not exogenous (i.e., generated by an independent process) but
is endogenous: it is produced by replaying the learner's own past,
potentially incorrect, output~(hypotheses). We formalise this by adapting the adversary in Littlestone’s mistake‑bound setting, and we refer to our framework as~\emph{Online Learning in the Replay Setting}. 

In our replay setting, after the learner outputs its hypothesis
$\hat{h}_t$, the environment may reveal either the ground-truth label
$f^*(x_t)$ or a replay $\hat{h}_i(x_t)$ for some $i < t$. The learner does not know which it received. The
goal is to  minimise the number of \emph{true-label mistakes}, i.e., mistakes in rounds where \(f^*\) is revealed. We say a
hypothesis class is learnable in the replay setting if there exists an
algorithm which makes a finite number of true-label mistakes even when
the game is played for infinite number of rounds.
We show that there exists a fundamental separation of the replay setting from classical online learning~\citep{littlestone1998}. There exists a simple
hypothesis class (see~\cref{ex:non-int}) on a domain of size \(N\) that is~\emph{properly} online learnable with finite mistakes but a replay adversary forces  any proper learner to incur \(\Omega\br{T}\) true-label mistakes. In the improper setting, the halving algorithm yields \(\bigO{\log N}\) mistakes but in the replay setting we show that any improper learner must incur \(\Om{N}\) mistakes.

Given this inherent separation from classical online learning, which is characterised by the Littlestone dimension~\citep{littlestone1998}, it is natural to ask whether there are structural or combinatorial properties that characterise learnability in the replay setting. We show that for proper learners, \emph{intersection-closedness} and \emph{Threshold dimension}~\citep{shelah1978ClassificationTheoryNumbers} capture the entire story: if a hypothesis class \(\HC\) is intersection-closed with finite Threshold dimension, then it is learnable in the replay setting~(\Cref{thm:general-adver}), and if it is not~(up to some transformations), every proper learner is doomed to make infinitely many true-label mistakes~(\Cref{thm:general-lower-adv-int-clos}). This quantitative result immediately yields a qualitative separation from classical online learning. For all intersection-closed hypothesis classes $\HC$ with \(VC\) dimension \(\vcdim\), the proper online mistake bound is known to be \(\bigO{\ldim{\HC}\vcdim\log\br{\vcdim}}\)~\citep{hanneke2021online,bousquet2020proper} while any proper replay learner must incur at least $\Omega\bigl(\tdim{\mathcal{H}}\bigr)$ mistakes. For improper learners in the replay model, we define a new complexity measure called the~\emph{Extended Threshold dimension}~(\Cref{defn:exthdim}) that yields matching upper and lower bounds for improper learnability of all hypothesis classes in the replay setting. 

\noindent\textbf{Paper Outline:} A brief summary of our main results is given in~\Cref{tab:bound}, and an overview of related work appears in~\Cref{sec:rel-work}.~\Cref{sec:setting} provides a formal technical introduction to the problem and~\Cref{sec:threshold} provides matching upper and lower bounds for the class of thresholds for both adaptive~(\Cref{thm:thresh-adap}) and stochastic~(\Cref{thm:thresh-stoch}) adversaries. In~\Cref{sec:general}, we provide our main results for general hypothesis classes.~\Cref{sec:int-close} introduces concepts such as the Threshold and Extended Threshold dimension~(\Cref{defn:threshdim,defn:exthdim}) and~\Cref{sec:adaptive_ub,sec:stochastic} lists our main result against adaptive and stochastic adversaries respectively. Finally,~\Cref{sec:discussion} ends with a discussion and open questions.

\subsection{Related Work} 
\label{sec:rel-work}

 \noindent \textbf{Online Learning} has been extensively studied under a variety of assumptions. Adaptive adversaries choose $x_t$ after observing the past~\citep{littlestone1998}, whereas stochastic adversaries draw $x_t\stackrel{\mathrm{i.i.d.}}{\sim}\cD$ for an unknown distribution $\cD$ fixed in advance~\citep{rakhlin2011OnlineLearningStochastica}. A related setting is the classical realisable mistake bound model of online learning which assumes every example can be correctly labelled by some unknown \(f^\ast \in \HC\). Performance in this model is measured by the total number of mistakes the learner incurs. The Littlestone dimension $\ldim{\HC}$ exactly characterises \emph{improper} learnability with an optimal mistake bound $\Theta\br{\ldim{\HC}}$. For proper learners, the bound is $\Theta\br{k\ldim{\HC}}$, where $k$ is the Helly number of $\HC$~\citep{bousquet2020proper,hanneke2021online}. The agnostic model drops the realisable assumption~\cite{ben2009agnostic,frenay2013classification,el2010foundations,wiener2011agnostic} and instead measures regret against the best hypothesis in hindsight. In the replay setting, we introduce a restricted noise model: the adversary may output the prediction of an earlier hypothesis of the learner and we measure performance using a variation of the mistake bound.

\noindent\textbf{The Closure Algorithm} is a well known algorithm \cite{helmbold1990learning, haussler1994predicting,auer1998OnlineLearningMalicious} and is central to the replay setting. It is also one of the few consistent learning algorithms known to achieve the optimal sample complexity for PAC learning of intersection-closed classes \cite{auer2007NewPACBound,darnstadt2015OptimalPACBound}. As any hypothesis class embeds into an intersection-closed class via its closure, understanding these structures can produce learners that generalise beyond this family itself \cite{rubinstein2022UnlabelledSampleCompression}. The same algorithm also underlies disagreement-based abstention~\citep{rivest1988learning,geifman2017selective,goel2023AdversarialResilienceSequential} and stream–based active learning \citep{cohn1994improving,hanneke2009theoretical,hanneke2014theory,hanneke2015minimax,hanneke2016refined,gelbhart2019relationship}. In the replay setting the learner must, in addition, preserve consistency within all future disagreement regions to distinguish true labels from replays.

\looseness=-1\noindent\textbf{Performative Prediction} studies the influence of a model's predictions on future outcomes~\citep{perdomo2020PerformativePrediction,miller2021OutsideEchoChamber,hardt2025PerformativePredictionFuture}. Unlike classical concept drift~\citep{kuh1990learning,bartlett1992learning}, the drift here is endogenous, i.e., induced by the learner’s own predictions. While performative prediction allows a gradual drift in the data distribution, the replay model differs in two respects: the data covariates remain unaffected and only the label depends, possibly adversarially, on past outputs.

\section{Technical Introduction via the Threshold Class}
\label{sec:prelim} 

\begin{table}[t]
    \centering
    
 \centering\small
    \begin{tabular}{%
        L{4cm}  C{4cm}  C{7cm}   
    }
      \toprule
      \textbf{Hypothesis Class}
        & \textbf{Adaptive Adv.\ }\(\cM_T\)
        & \textbf{Stochastic Adv.\ }\(\bE\bs{\cM_T}\) \\
      \midrule
      Thresholds on \([N]\) &  \(\Theta\bigl(\min\{N,T\}\bigr)\)(Thm~\ref{thm:thresh-adap})
          &\(\Theta\bigl(\min\{N,\log T\}\bigr)\) (Thm~\ref{thm:thresh-stoch}) \\
      \midrule
        \multirow{2}{3.3cm}{Intersection-Closed~\(\HC\)}
          & \multirow{2}{3cm}{\centering\(\Theta\bigl(\tdim{\HC}\bigr)\)(Thm~\ref{thm:general-adver})}
          & \(O\bigl(\min\{\tdim{\HC},\,\vcdim(\HC)\,\log T\}\bigr)\) \multirow{2}{1cm}{~(Thm~\ref{thm:general-stoch})} \\
          &
          & \(\Omega\bigl(\min\{\tdim{\HC},\,\log T\}\bigr)\) \\
          \midrule
        General~\(\HC\)
          & \(\Theta\bigl(\extdim{\HC}\bigr)\) (Thm~\ref{thm:general-adver})
          & \(\Omega\bigl(\min\{\extdim{\HC},\, \log T\}\bigr)\) (Thm~\ref{thm:general-stoch}) \\
      \bottomrule
    \end{tabular}
    \label{tab:bounds}

    \caption{Let $\cM_T$ be the total number of mistakes up to time $T$. We give upper (\(O(\cdot)\)) and lower (\(\Omega(\cdot)\)) bounds on $\cM_T$ for adaptive and stochastic adversaries, where \(\Theta(\cdot)\) denotes matching upper and lower bounds.}
    \label{tab:bound}\vspace{-10pt}
\end{table}

\noindent\textbf{Notation.} %
We write \( [N] \coloneqq \{1, 2, \ldots, N\} \) for \( N \in \N^+ \).
We use \(2^{\cX}\) to denote \(\{h:\X\rightarrow \{0,1\}\}\), the set of all binary functions on a domain set \(\cX\).
For any set \( A \subseteq \mathcal{X} \), we define the indicator function \( \mathbb{I}\{A\} : \mathcal{X} \to \{0,1\} \) by \( \mathbb{I}\{A\}(x) =\mathbb{I}\{x\in A\}= 1 \) if \( x \in A \), and \( 0 \) otherwise.
For any two functions \(h,f : \mathcal{X} \to \{0,1\}\), the $f$-representation of \(h\) is $ h^{f}:= \bigl\{\,x \in \mathcal{X} \mid h(x) \neq f(x)\bigr\}$, following the notation from \cite{ben-david20152NotesClasses}.
For $f\equiv0$, the set coincides with the function $h$ itself. Throughout, we will use the set and function notation of $h$ interchangeably, depending on the context.

\subsection{Problem Setting}
\label{sec:setting}
We begin by recalling the definition of the standard \emph{mistake bound model}.
\begin{defn}[Mistake Bound Model]
\label{def:online-learning}
    Let \(T \in \bN^+\) be the number of rounds, \(\cX\) be the domain set and \(\cH\subseteq 2^{\cX}\) be a hypothesis class. We define an online game in the \emph{mistake bound model} between a learner $\cA$ and the Nature as follows: For each round $t\in[T]$, Nature produces a sample $x_t\in\X$ and the learner $\cA$ selects a hypothesis $\hat{h}_t\in  2^{\cX}$ and predicts $\hat{y}_t=\hat{h}_t(x_t)$. Nature plays the label $y_t$ and reveals it to $\cA$. The learner incurs a loss of $\bI\bc{y_t\neq\hat{y}_t}$. We assume the realisable setting, i.e. there exists $f^*\in\HC$ such that $y_t=f^*(x_t)$ for all $t\in[T]$.
\end{defn} Building on this framework, we introduce the \emph{replay setting}, a variant of the mistake bound model in which the learner’s past predictions influence the feedback received in future rounds.

\begin{defn}[Replay setting]
\label{def:replay-setting}
Let \(T\in\N^+\) be the number of rounds, \(\cX\) be the domain set and \(\cH\subseteq 2^{\cX}\) be a hypothesis class. We define an online game in the \emph{replay setting} between a learner \(\cA\) and the Nature in the following way:
\end{defn}
\begin{defbar}
\textbf{Online Learning in the Replay Setting}\\
For \(t = 1, \ldots, T\):
\begin{itemize}[leftmargin=*,itemsep=0em,topsep=1pt]
        \item Learner \(\cA\) outputs a hypothesis \(\hat h_t \in 2^{\cX}\);
        \item Nature (adversarial or stochastic) produces a sample \(x_t \in \cX\) and reveals \((x_t, y_t)\) to \(\cA\),
        \begin{equation*}
            y_t=\begin{cases}
            f^*(x_t) & \text{ for some }f^*\in\HC \quad \text{ or}\\
            \hat{h}_i(x_t) & \text{ for some  } i< t %
            \end{cases}
        \end{equation*} 
        \item Learner suffers loss \(\bI\bc{y_t \neq \hat{h}_t(x_t) \text{ and }y_t = f^\ast(x_t)}\), \emph{not revealed} to the learner.
\end{itemize}
Nature picks \(f^\ast \in \cH\) consistent with all rounds $t$ for which \(y_t \notin \bc{\hat h_1(x_t), \ldots, \hat h_{t - 1}(x_t)}\).
\end{defbar}
\noindent Let \(\hat y_t \coloneqq \hat h_t(x_t)\) for \(t \in [T]\). The number of mistakes is computed at the end of the game as follows: 
    \begin{equation}\label{eq:mistakes}
            \cM_T(\cA)  = \sum_{t = 1}^T \bI\bc{y_t \neq \hat y_t \text{ and }y_t = f^\ast(x_t)}
    \end{equation}
\noindent
We refer to the Nature in the above definition as the \emph{replay adversary}. 

The replay setting is a generalisation of the mistake bound model: 
since in~\Cref{def:online-learning}, only the value \(\hat h_t(x_t)\) is used, the learner can equivalently output the hypothesis \(\hat h_t\) \emph{before} the Nature produced a particular sample \(x_t \in \cX\), as adopted in~\Cref{def:replay-setting}.
When the replay adversary reveals $y_t=f^*(x_t)$ in every round $t\in [T]$, we recover~\Cref{def:online-learning}.
Additionally, the replay adversary can label $x_t$ using an earlier hypothesis $\hat{h}_i(x_t)$ for some $i<t$, leading to a more powerful adversary. The learner in the replay setting does not observe whether the sample \(x_t\) was labelled by $f^*$ or a replay hypothesis and therefore does not have access to the number of mistakes made up to step \(t \in [T]\), as defined in~\cref{eq:mistakes}. Otherwise, the learner could have simply disregarded the labels produced by replay hypotheses, reducing the setting to the mistake bound model. 

    Furthermore, the replay adversary is less powerful than the agnostic online learning setting (\cite{ben2009agnostic}), where there is no restriction on how the samples \(x_t\) are labelled by the Nature and the performance of the learner is measured against the best hypothesis in hindsight from \(\cH\). We now recall classical notions from online learning, which characterise both the generation of input samples \( x_t \in \mathcal{X} \) and the restrictions on the learner's hypothesis selection.
\begin{defn}[Adaptive and Stochastic Adversaries]
    \looseness=-1 Consider the replay setting in~\Cref{def:replay-setting}.
    We call the adversary \emph{adaptive} in the case when the Nature selects an arbitrary sample \(x_t \in \cX\) at each round \(t \in [T]\), after observing the history: \(x_1, \ldots, x_{t - 1}\) and \(\hat h_1, \ldots, \hat h_{t - 1}\).
    The case when the Nature at each round \(t\) samples \(x_t \overset{\iid}{\sim} \cD\) for some fixed distribution \(\cD\) over \(\cX\) defines a \emph{stochastic adversary}.
\end{defn}

\begin{defn}[Proper and improper learner]
    A learner \(\cA\) that for each \(t \in [T]\)  outputs \(\hat h_t \in \cH\), is called \emph{proper}. Learners that are not proper are called \emph{improper}.
\end{defn}

In online learning, a \emph{version space} represents the learner's current knowledge of $f^*$. It is the subset of all $h\in\HC$ such that $h$ is  consistent with all previously observed samples. The replay setting requires a slight variation of this concept; we do not seek consistency with \textit{all} previously observed samples (as some of them could be labelled by replay hypotheses), but only with those which the learner is sure are labelled by $f^*$. We therefore introduce the concept of a \textit{reliable} version space:

\begin{defn}[Reliable Version Space] \label{defn:version_space}
Let $\bc{(x_1,y_1),\dots,(x_T,y_T)}$ be the samples observed by the learner, and $\HC_{t}\coloneqq\{\hat h_1,\dots, \hat h_{t}\}$ be the set of produced hypotheses until step \(t \in [T]\). Define the set of indices corresponding to the labels that do not agree with any previously produced hypothesis:
set $I_0=\emptyset$ and for \(t \in [T]\) set $I_{t}=I_{t-1}\cup\{t\}$ if for all $h'\in\HC_{t-1}: h'(x_{t})\neq y_{t}$. Otherwise, set $I_{t}=I_{t-1}$.\\
The \emph{reliable version space} is defined as the set of hypotheses in $\mathcal{H}$ that are consistent with the samples from $I_t$. For \(t \in [T]\),
\begin{equation*}
    \VS_t:=\bc{h\in\HC \mid \forall i\in I_t: h(x_i)=y_i}=\HC|_{\bc{(x_i,y_i)_{i\in I_t}}}.
\end{equation*}
\end{defn}

By construction it holds that $\VS_1 \supseteq \VS_2\supseteq \dots \supseteq \VS_T\ni f^*$. Informally, the set \(I_t\) is the set of indices of labels that could not have been replayed. We remark that if the replay adversary only reveals labels from replay hypotheses, the learner may not accrue any true-label mistakes $\cM_t(\cA)=0$ while the reliable version space remains the full hypothesis class \(\VS_t = \cH\). In this case,  the learner does not learn anything about \(f^\ast\) while remaining mistake-free. Next, we provide results for online learning in the replay setting for a canonical but important hypothesis class: one-dimensional thresholds. In later sections, similar to other works~(e.g.~\citep{alon2019private}), we use these results to generalise to more complex hypothesis classes.

\subsection{Learning Thresholds}
\label{sec:threshold}
Let the domain be $\X = [N]$ and define the threshold class $\HC_{\text{thresh}}=\{f_k:\X\rightarrow\{0,1\}\mid k\in[N] \cup \bc{0}\}$ with \(f_0 \equiv 0\) and \(f_k(x) = \bI\bc{x \geq k}\) for \(k \in [N]\).
We describe a simple algorithm that learns $\HC_{\text{thresh}}$ in the replay setting.
\begin{defn}[Conservative Threshold Strategy]\label{alg:conservative_thresh}
Let $T\in\N$ be the time horizon.
The learner initialises with $\hat h_1\equiv 0$ and subsequently only updates on a false negative: at time \(t\in[T]\), if \(y_t = 1\) and \(\hat{h}_t(x_t) = 0\), it sets \(\hat{h}_{t+1} = \bI\bc{ \cdot \geq x_t}\in\HC_{\text{thresh}}\): otherwise, it keeps \(\hat{h}_{t+1} = \hat{h}_t\).
\end{defn}

We refer to this algorithm as a \emph{conservative algorithm}, since it only updates upon making a mistake and the new output hypothesis is the minimal one in the class $\HC_{\text{thresh}}$ that rectifies the mistake. We give results on upper and lower bounds on $\mathcal{M}_T(\mathcal{A})$, showing that the conservative algorithm is optimal against both adaptive and stochastic adversary,
up to constant factors.

\begin{restatable}[Adaptive Adversary on \(\cH_{\text{thresh}}\)]{thm}{advthresh}\label{thm:thresh-adap} 
Let $\cA$ be the learner in~\Cref{alg:conservative_thresh}. Against any adaptive adversary, \(\cA\) incurs at most $\cM_T(\cA)=\bigO{\min\{N,T\}}$  mistakes.
Conversely, for all learners $\cA'$, there exists an adaptive adversary that forces  $\cM_T(\cA')=\Om{\min\bc{N,T}}$ mistakes.
\end{restatable}

\begin{proof}[Proof sketch](Full proof in~\Cref{app:adap_threshold})
    For the upper bound, observe that the learner $\mathcal{A}$ only errs on positive samples $y_t = 1$ when it instead predicts $\hat{y}_t = 0$, and the learner can make at most $N$ such mistakes. For the lower bound, consider first the learners that start with \(\hat h_1 \equiv 0\).

    \noindent
    If for some $t \in [T/2]$ there exists \(x^* \in \cX\) such that \[\exists i\neq j<t: \{\hat h_i(x^*),\hat h_j(x^*)\}=\{f(x^*),f'(x^*)\}=\{0,1\}\] with $f,f'\in\VS_t$ and $\hat h_i, \hat h_j$ reused hypotheses,
    then in the following rounds it is impossible for the learner to learn the correct label of \(x^\ast\). In that case, an adaptive adversary can force the learner to make linear in $T$ number of mistakes, by querying $x^*$ at every round.\\ 
    Otherwise, an adaptive adversary can force $N$ mistake by querying: $((N, 1), \dots, (1, 1))$.
    
    \noindent
    For a learner that starts with  some other hypothesis $\hat h_1 \in \bc{0, 1}^{\cX}$, let \(P = \bc{x\in\X\mid \hat h_1(x)=1}\) and $k':=|P|$. 
    In that case, the replay adversary only gives samples with label 1 from \(\cX \setminus P\). Repeating this argument with an adversary giving samples $x\in P$ with label $1-\hat h_1(x)$, shows that the learner makes $\min\bc{\max\bc{N-k',k'},T} = \Omega(\min \bc{N, T})$ mistakes.
\end{proof}

\begin{restatable}[Stochastic Adversary on $\cH_{\text{thresh}}$]{thm}{stochthresh}\label{thm:thresh-stoch}
Against any stochastic adversary, the expected number of mistakes by the algorithm in~\Cref{alg:conservative_thresh} is $\bE\bs{\cM_T(\cA)}=\bigO{\min\{N,\log T\}}$.
Conversely, for all \(\cA'\), there exists a stochastic adversary that forces $\bE\bs{\cM_T(\cA')}=\Om{\min\{N,\log T\}}$ mistakes.
\end{restatable}
\noindent The proof is given in~\Cref{app:stoch_thresholds}.

\section{Algorithm for General Classes}
\label{sec:general}

In this section, we present a general algorithm to learn hypothesis classes in the replay setting. We begin by introducing some definitions.
\subsection{Intersection Closedness and Extended Threshold dimension}
\label{sec:int-close}
\begin{defn}[$\mathcal{H}$-closure]\label{defn:closure}
Let $\cH\subseteq 2^{\cX}\}$ be a hypothesis class. Define the \emph{closure under $\HC$}, denoted by \(\clos_\mathcal{H}\):
$\clos_\mathcal{H}(Y):=\bigcap_{h\in\mathcal{H}:Y\subseteq h}h$ for any $Y \subseteq \cX.$ 
We also define the class of all intersections of hypotheses in $\HC$ as $\HCC := \{ \bigcap_{h \in S} h \mid \emptyset \neq S \subseteq \HC \}$ and note that \(\clos_{\HC} = \clos_{\HCC}\).
\end{defn}

\begin{defn}[Intersection-closed]\label{defn:intersection-closed}
 A hypothesis class $\HC$ is called \emph{intersection-closed over arbitrary intersections}\footnote{The standard definition of (finite) intersection-closed classes is as follows: for all $h_1,h_2\in\HC$: $h_1\cap h_2\in\HC$. Note that for finite $\X$ the two notions of finite and arbitrary intersection-closed coincide. However, the replay setting necessitates $\clos_{\HC}(Y)\in\HC$ for all $Y\subseteq\X$ and therefore we impose the stronger notion of (arbitrary) intersection-closeness.} if for all  $S\subseteq\HC$, one has $\bigcap_{h \in S} h\in\HC$. In the following, we refer to such a class $\HC$ just as intersection-closed.
\end{defn}
Note that if $\HC$ is intersection-closed, then for all $Y\subseteq \X$, it holds that $\clos_\HC(Y)\in\HC$ and furthermore $\HC=\HCC$.
We denote $h_{\min}:=\clos_{\HC}(\emptyset)= \bigcap_{h \in \HC}h$ and for $\HC$ intersection closed, one has $h_{\min}\in\HC$. \\

\begin{defn}[Threshold dimension,~\cite{shelah1978ClassificationTheoryNumbers}]\label{defn:threshdim}
    Let \(\cH \subseteq 2^{\cX}\) be a hypothesis class. We define \emph{Threshold dimension} of \(\cH\), denoted by \(\tdim{\cH}\), as the largest \(k\), such that there exist \(x_1, \ldots, x_k \in \cX\) and \(h_0, h_1, \ldots, h_k \in \cH\) with \(h_i(x_j) = \bI\bc{j \leq i}\). The set of points \(x_1, \ldots, x_k\) and set of hypotheses \(h_0, h_1, \ldots, h_k\) are called \emph{the witness sets} for \(\cH\). 
\end{defn}
The Threshold dimension originated in set theory and is often used to generalise results from threshold functions to general hypothesis classes \cite{alon2019private}. As shown in \cite{shelah1978ClassificationTheoryNumbers,hodges1997ShorterModelTheory}, the Threshold dimension can be related exponentially to the Littlestone dimension $\Ldim(\HC)$, which characterises online learning in the mistake bound model: $\tdim{\HC} = \Omega(\log(\Ldim(\HC))) \text{ and } \Ldim(\HC) = \Omega(\log(\tdim{\HC}))$. Next, we introduce a new complexity measure, which characterises learnability in the replay setting.

\begin{defn}[Extended Threshold dimension]\label{defn:exthdim}
    Let \(\cH \subseteq 2^{\cX}\) be a hypothesis class and let $\HC^f$ denote its $f$-representation. We define the \emph{Extended Threshold dimension} of \(\cH\):
    \begin{equation*}
        \extdim{\cH} \coloneqq \min_{f \subseteq \cX} \tdim{\overline{\cH^f}}.
    \end{equation*}
\end{defn}
\begin{restatable}{prop}{exthreshbounds}\label{prop:exthresh_bounds}
Let \(\cH \subseteq 2^{\cX}\) be a hypothesis class. Then, the following holds:
\begin{equation*}
    \tdim{\cH} / 2\leq  \extdim{\cH} \leq \tdim{\HCC} \leq \abs{\HC}.
\end{equation*}
In particular, if $\HC$ is intersection-closed, then \(\extdim{\cH} \leq \tdim{\cH}\). Furthermore, for any \(N \geq 2\), there exists a hypothesis class \(\cH^\ast\) over the domain \(\cX = [2N]\), such that 
$\tdim{\cH^\ast} = 3$, but $\extdim{\cH^\ast} \geq N$.
\end{restatable}
Additional results on when \(\extdim{\HC}\) gives a substantial improvement over considering $\tdim{\HCC}$ can be found in \Cref{sec:ext-thresh-app}.

\begin{defn}[\Gls{trap}]\label{defn:trapt}
    For \(t \in [T]\), define a \emph{\gls{trap}} at time $t$ as follows:
    \begin{equation*}
        \Trp_t:=\Bigl\{x\in\X\mid \exists f, f'\in\VS_t \text{ and }  h, h' \in \HC_{t-1} \text{ s.t. } \bc{f(x), f'(x)} = \bc{h(x), h'(x)} = \bc{0, 1}\Bigr\}. 
    \end{equation*}
\end{defn}
Informally, $x\in\Trp_t$ means that, by round $t$, the learner has already predicted both labels on $x$, and $\VS_t$ still contains hypotheses that realise each of those labels. At such a point the adversary can replay either label indefinitely, and the learner has no way to tell replay from truth. 
We show that any learner that admits a non-empty \gls{trap} will incur linear error in $T$ in both the adaptive and stochastic adversarial setting. We refer the reader to \Cref{sec:trap-region-app} for formal details, and state an immediate corollary of these results for now.

\begin{corollary}\label{cor:sublinear-notrap}
    Any learner $\mathcal{A}$ with sub-linear error in $T$, i.e., $\mathcal{M}_T(\mathcal{A})=o(T)$ or $\bE\bs{\cM_T(\cA)}=o(T)$, must satisfy $\Trp_t = \emptyset$ for all $t \in [T/2]$.
\end{corollary}

\subsection{Closure Algorithm}

We define the general algorithm  called the \emph{closure algorithm} in \Cref{alg:closure_explicit}. This algorithm is used to attain the upper bounds in this work.

\begin{algorithm}[t]
\caption{Closure Algorithm}
\label{alg:closure_explicit}
\begin{algorithmic}[1]
\State \textbf{Input:} Hypothesis class $\HC \subseteq 2^\X$, time horizon $T \in \N^+$, representation \(f \subseteq \cX\)
\State Initialise $\hat{h}_1 \gets h_{\min}:=\bigcap_{h \in {\HC^f}} h\in \overline{\cH^f}$

\For{$t = 1$ \textbf{to} $T$}
    \State Learner outputs $\hat{h}_t^f$ 
    \label{step:flip-hypo}
    \State Learner receives \((x_t, y_t) \in \cX \times \bc{0, 1}\)
    \If {\(x_t \in f\)} set \(y_t \gets 1 - y_t\)\label{step:flip-label}
    \EndIf
    \If{$y_t = 1$ and $\hat{h}_t(x_t)=0$}
    \State $\hat{h}_{t+1} \gets  \clos_{\HC^f}(\hat{h}_t \cup \{x_t\})\in\overline{\cH^f}$ \label{step:hyp-close}\Comment{Smallest $S\in\overline{\cH^f}: \hat{h}_t \cup \{x_t\}\subseteq S$}   
    \Else{ $\hat{h}_{t+1} \gets \hat{h}_t$}
    \EndIf
\EndFor
\end{algorithmic}
\end{algorithm}
Our procedure takes as input a function $f$ to realise an $f$-representation of $\HC$. When $f$ can be chosen so that $\HC^f$ is intersection-closed, the closure algorithm enjoys a substantially improved mistake bound. Deciding (or constructing) such an $f$ is nontrivial in the general case: to the best of our knowledge,  we do not know how to test for existence or compute. For classes of VC dimension~1, however, we show that a suitable $f$ always exists and can be found efficiently (\Cref{app:f_for_vcd1}). Extending this guarantee to general classes remains open.

\begin{remark}\label{rem:f-rep}

     Since the $f$-representation of a class $\mathcal{H}$ is a bijection with $(\HC^f)^f=\HC$, we may choose  $f = \arg\min_{f'\subseteq \X} \tdim{\overline{\mathcal{H}^{f'}}}$ and then learn the transformed class \(\mathcal H'=\mathcal H^{f}\). After this choice, we can treat \(f\equiv 0\) in the algorithm so the learner never has to flip predictions or observed labels (lines~\ref{step:flip-hypo} and \ref{step:flip-label} of~\Cref{alg:closure_explicit}).
\end{remark}

The learner $\cA$ following the closure algorithm maintains a current prediction $h_t \in \HCC$. This hypothesis is updated conservatively using the $\clos_{\HC}$ operator to ensure consistency with observed true labels while remaining in $\HCC$ at all times (this can be seen by a simple induction argument and the fact that $\HCC$ is intersection-closed). We initialise the algorithm with the minimal element in closure $h_{\min}\in\HCC$. Note that when $\HC$ is the class of thresholds over finite $\X$ we recover the conservative threshold strategy as described in~\Cref{alg:conservative_thresh}.
Depending on whether $\HC$ is intersection-closed,~\Cref{alg:closure_explicit} is a \emph{proper} learner as $\HC=\HCC$ for intersection-closed $\HC$ or an \emph{improper} learner if $\HC$ is not intersection-closed, outputting $\hat h_t\in \HCC\supsetneq\HC$.

\label{sec:adaptive_ub}
\begin{restatable}[Adaptive Adversary]{thm}{advgeneral}\label{thm:general-adver}
    Let $\HC\subseteq2^\X$ be a hypothesis class and let \(\cA\) be the learner following~\Cref{alg:closure_explicit}. 
    Then, \[\cM_T(\cA) \leq \extdim{\cH}.\] 
    Furthermore, for any learner \(\cA'\) of \(\cH\) and time horizon \(T \geq \extdim{\HC}\), there exists an adversary such that \[\cM_T(\cA')=\Om{\extdim{\cH}}.\]
\end{restatable}

\begin{proof}[Proof sketch]
(Full proof in \Cref{sec:clos-alg-app}) 
The proof of the upper bound is similar to the threshold case in~\Cref{thm:thresh-adap}, and follows from the properties of~\Cref{alg:closure_explicit}.
We note that for a given \(f \subseteq \cX\), the closure algorithm \(\cA\) can be implemented efficiently, achieving \(\cM_T(\cA) \leq \tdim{\overline{\cH^f}}\). 
Achieving the bound \(\cM_T(\cA) \leq \extdim{\cH}\) requires minimising over \(f \subseteq \cX\), which may be exponential-time.

For the lower bound, we also follow the argument in~\Cref{thm:thresh-adap}, with several generalisations. 
A straightforward way is to consider the witness sets of \(\cH^f\) (see~\Cref{defn:threshdim}), where \(f\) is the first hypothesis output by \(\cA'\). This proves that the adversary can impose at least \(\tdim{\cH^f}\) mistakes on such \(\cA'\).
We go one step further and construct an adversary that forces \(\cA'\) to make \(\tdim{\overline{\cH^f}}\) mistakes, which in general can be much larger (see~\Cref{prop:exthresh_bounds}). 
We show that the learners that do not follow the same strategy as the closure algorithm must have \(\Trp_t \neq \emptyset\) for some \(t \leq T/2\), which leads to \(\Omega(T)\) mistakes due to~\Cref{cor:sublinear-notrap}. 
\end{proof}

\subsection{Stochastic Adversary}
\label{sec:stochastic}

Next, we focus on learning under  \emph{stochastic adversary} for the replay setting. We show upper bounds for~\Cref{alg:closure_explicit} and contrast this with lower bounds that hold for any learner. 

\begin{restatable}[Stochastic Adversary]{thm}{stochgeneral}
\label{thm:general-stoch}
Let $\HC\subseteq2^{\X}$ be a hypothesis class with VC dimension
$\vcdim$ and $T\in\N^+$ be the time horizon.
When \(\cH\) is intersection-closed, there exists a learner $\cA$~(\Cref{alg:closure_explicit}) whose expected number of mistakes is \[\bE\bs{\cM_T(\cA)} = O\bigl(\min\bc{\tdim{\HC},\vcdim\log T}\bigr).\]
Furthermore, for any \(\cH\), there exists a distribution $\cD$ over $\X$ and $f^*\in\HC$ such that any deterministic learner $\cA'$ has expected error \[\bE\bs{\cM_T(\cA')}=\Om{\min\bc{\extdim{\HC},\log T}}.\]
\end{restatable}

The upper bound follows from an application of optimal PAC-learning results for closure algorithms~\citep{darnstadt2015OptimalPACBound}, which bound the expected error of a classifier trained on i.i.d.\ samples labelled by a fixed $f^* \in \HC$. For the lower bound, we reduce to the one-dimensional threshold case (as in \Cref{thm:thresh-stoch}); and use the Extended Threshold dimension to construct a worst case distribution supported on the the witness set. The proofs are provided in~\Cref{app:stoch-general}.

Although we state the upper bound to hold only for intersection-closed classes, a similar bound holds for general classes by the reasoning in \Cref{thm:general-adver}. The cost is that the VC term blows up to  $\vcd{\overline{\HC^f}}$ in place of $\vcd{\HC}$. This increase can be substantially large, leading to worse mistake bounds~(see~\Cref{app:f_for_vcd1}).
Note that for intersection-closed classes \(\extdim{\cH} = \Theta(\tdim{\cH})\)~(see~\Cref{prop:exthresh_bounds}), thus the above upper and lower bound for intersection-closed classes only have a multiplicative gap of $
\vcdim$. 
For general classes, we believe the current $\log T$ lower bound is not tight. We conjecture that an additional multiplicative factor depending on the width (the size of the largest antichain) of the class is necessary, complementing $\extdim{\HC}$ which can be seen as the depth of the class (the length of the longest chain). \\

We further generalise the result to certain classes with infinite VC and threshold dimension. We consider the class of all convex subsets of \(\reals^d\), which has infinite VC dimension but is intersection-closed (since the intersection of any collection of convex sets is convex). The closure algorithm in \Cref{alg:closure_explicit} is equivalent to simply computing the convex hull of the observed positive examples, which is the smallest convex set consistent with the data. This highlights the convex hull as a special case of a closure operator.

\begin{restatable}[Convex Bodies—Stochastic Adversary]{thm}{stochconvex}
\label{thm:convex-stoch}
Let \(d \in \mathbb{N}\), and define \(\HC_d = \{ C \subseteq \mathbb{R}^d : C \text{ is convex} \}\) as the class of all convex subsets of \(\mathbb{R}^d\).  
Let \(\mu\) be a probability distribution supported on a fixed convex set \(C^* \subseteq \mathcal{B}_1^d\), where \(\mathcal{B}_1^d\) denotes the unit ball in \(\mathbb{R}^d\).
Assume that \(\mu\) admits a density \(f\) which is upper bounded by some constant \(M<\infty\).  Then
there exists a learner $\cA$ (\Cref{alg:closure_explicit}), such that for any time horizon \(T\ge 1\) the expected number of mistakes satisfies 
\[\bE\bs{\cM_T\br{\cA}}=\bigO{\log T}\text{ if } d=1 
\quad \text{and} \quad \bE\bs{\cM_T\br{\cA}}=\bigO{T^{\frac{d-1}{d+1}}} \text{ if }d\geq 2.\]
Furthermore, there exists a distribution $\cD$ and target $f^*\in\HC$ such that any learner \(\cA'\) in the replay setting incurs \(\Omega(\log T)\) expected mistakes when \(d = 1\), and \(\Omega\br{T^{\frac{d-1}{d+1}}}\) when \(d \ge 2\), matching the upper bounds for the closure algorithm.
\end{restatable}

\noindent The proof is given in~\Cref{app:stoch-general}.

\subsection{Note on Proper vs. Improper Learners}
In this section, we show that if a class is not intersection-closed (up to an \(f\)-representation), then it is not properly learnable, whereas \Cref{alg:closure_explicit} still provides an improper learner by predicting in the closure $\HCC$. Another approach to \emph{properly} learn a non-intersection-closed class \( \HC \) is to find an alternative \(f\)-representation in which the class is intersection-closed. For instance, if $\HC$ is union‑closed, choosing $f\equiv 1$ produces an intersection‑closed class $\HC^{f}$. When $\HC^{f}$ is intersection closed, \Cref{prop:exthresh_bounds} yields  \( \tfrac12\,\tdim{\HC^{f}}\le \extdim{\HC}\le\tdim{\HC^{f}}, \) so running \Cref{alg:closure_explicit} on any such representation is optimal up to a factor of $2$.

\begin{restatable}[Intersection-closedness is sufficient and necessary for proper learnability]{thm}{intersectionadv}
\label{thm:general-lower-adv-int-clos}
Let \(\cH\subseteq 2^{\cX}\) be a hypothesis class. Then, $\HC$ is properly learnable in the replay setting, if and only if there exists an $f$-representation of the class \(\HC\) such that $\HC^f := \{h^f \mid h\in\HC\}$ is intersection-closed.
\end{restatable}

The proof is given in~\Cref{app:proper-int-closed}. Using this result, we get that if $\HC$ is not intersection-closed under any \(f\)-representation, a proper learner in the replay setting makes unbounded mistakes. This contrasts with classical online learning, where the proper mistake bound is upper and lower bounded~\citep{hanneke2021online} by \(\Theta\br{\ldim{\HC}\mathrm{Helly}\br{\HC}\log\br{\mathrm{Helly}\br{\HC}}}\)\footnote{Introduced in ~\citep{helly1923UberMengenKonvexer}, see e.g. \cite{bousquet2020proper} for a definition of Helly number} and is finite for every finite class~\citep{bousquet2020proper}.  Despite the negative result on proper learnability,~\Cref{thm:general-adver} shows that the improper learner in \Cref{alg:closure_explicit} incurs at most \(\extdim{\HC}\) (which is finite for finite hypothesis classes) mistakes even on non-intersection closed hypothesis class at the expense of being improper.

\begin{example}\label{ex:non-int}
    Consider $\HC=\bc{\mathbb{I}\bc{[a,b]}\cup\mathbb{I}\bc{[c,d]}\mid 0\leq a<b\leq c<d\leq 1}$ the hypothesis class of union of two intervals. An adaptive adversary can force any proper learner starting at $h_{\min}\equiv 0$ to make true label mistakes growing linearly in $T$: after having observed 3 samples labelled with a (true) 1, a \emph{proper} learner must output a hypothesis which predicts on a region it is uncertain of. No matter which choice the learner makes, it holds $\Trp_t\neq\emptyset$ and by \Cref{thm:no_trap} the adversary can then force linear mistakes. Note that the \emph{improper} learner given in \Cref{alg:closure_explicit} outputs a hypothesis from $\HCC$ consisting of singletons and thus incur at most $\abs{\cX}+1$ mistakes.
\end{example}

The above example shows a class that is not \emph{properly} learnable in the replay model but is \emph{improperly} learnable with at most \(\abs{\cX}+1\) mistakes. 
In contrast, it is easy to check that improper learnability in the standard online learning using the halving algorithm incurs only~\(\bigO{\log \abs{\cX}}\) mistakes.~\Cref{prop:exthresh_bounds} gives further classes where the
Littlestone dimension, and hence the classical improper bound, is constant, while the Extended Threshold dimension, and thus the replay
mistake bound, grows with~$\abs{\X}$.

\section{Discussion and Open Problems}
\label{sec:discussion}

In this paper we formalise a new learning setting in which the learner must ensure it is not misled by its own earlier outputs. This captures modern reality, where a growing share of training data is generated by past models. By distilling this phenomenon into a clean learning-theoretic framework, we provide a foundation for analysing the error feedback loops that are increasingly shaping machine learning pipelines. We showed that replay makes learning qualitatively harder than the classical mistake bound model. For example, thresholds on an \(N\) point domain admit a \(\bigO{\log N}\) proper learner in the classical setting but any proper learner in the replay setting must incur \(\Omega\br{N}\) mistakes. More generally, for any intersection-closed class \(\HC\), the proper replay mistake bound is \(\Omega\br{\tdim{\HC}}\), whereas the classical mistake bound for proper online learning is \(\bigO{\ldim{\HC}\vcdim\log\br{\vcdim}}\)~\citep{bousquet2020proper}.\\

\noindent\textbf{Open questions}
First, we know that every finite class 
admits finite~\(\extdim{\HC}\) and can thus be learned improperly in the replay setting with finitely many mistakes. Yet, we lack a precise characterisation of \(\extdim{\HC}\) for infinite hypothesis classes as well as tighter bounds in terms of other complexity measures beyond intersection-closed classes. Additionally, the question of tight lower bounds for the improper stochastic case also remains open. Second, our results show that if a non-intersection-closed class admits an $f$-representation~\citep{ben-david20152NotesClasses} that makes it intersection-closed, we can properly 
learn the class in the replay setting. We provide a characterisation for classes with VC dimension 1 in \Cref{app:f_for_vcd1}. However, to the best of our knowledge, it remains open when such an $f$-representation exists for general classes. Reconciling these questions along with settling the gap between the upper and lower bounds in~\Cref{thm:general-stoch} remain interesting directions for future work.\\

\noindent\textbf{Extensions of the Replay model} Real models rarely behave like our worst-case adversary in the replay setting. First, modern generative systems are stochastic: for example, they estimate a probability \(h(x)\) and then sample a label \(y \sim \mathrm{Bern}(h(x))\). In a variant to our problem, we could allow the learner do the same. The replay adversary is only allowed to sample from \(\mathrm{Bern}(h(x))\) and cannot choose with full precision which label to show. Second, in practice, auto-labelling mechanisms may require consensus among multiple labellers. A natural extension therefore is to consider the case when a replay label becomes eligible only after it has appeared in at least \(k\) prior hypotheses. Finally, real datasets typically contain a small audited fraction of ground-truth labels e.g. by assuming that the amount of times the adversary shows $f^*$ is fixed and known to the learner. We hope that resolving these questions and relaxing the problem setting to more closely match practical concerns will bring more theoretical insight while also delivering concrete value to modern machine learning systems.

\section*{Acknowledgement}
    AS acknowledges the Novo Nordisk Foundation for support via the Startup grant (NNF24OC0087820) and VILLUM FONDEN via the Young Investigator program (72069). The work was completed while DD was partly supported by the ETH
AI Center and the ETH Foundations of Data Science initiative. The authors would also like to thank Amir Yehudayoff for very insightful discussions.
\newpage
\bibliographystyle{alpha}
\bibliography{refs}

\newcommand{\etalchar}[1]{$^{#1}$}
\begin{thebibliography}{ACRL{\etalchar{+}}24}

\bibitem[AC98]{auer1998OnlineLearningMalicious}
Peter Auer and Nicol{\`o} {Cesa-Bianchi}.
\newblock On-line learning with malicious noise and the closure algorithm.
\newblock {\em Annals of Mathematics and Artificial Intelligence}, 1998.

\bibitem[ACRL{\etalchar{+}}24]{alemohammad2024selfconsuming}
Sina Alemohammad, Josue Casco-Rodriguez, Lorenzo Luzi, Ahmed~Imtiaz Humayun, Hossein Babaei, Daniel LeJeune, Ali Siahkoohi, and Richard Baraniuk.
\newblock Self-consuming generative models go {MAD}.
\newblock In {\em {International Conference on Learning Representations~(ICLR)}}, 2024.

\bibitem[ALMM19]{alon2019private}
Noga Alon, Roi Livni, Maryanthe Malliaris, and Shay Moran.
\newblock Private pac learning implies finite littlestone dimension.
\newblock In {\em Proceedings of the 51st Annual ACM SIGACT Symposium on Theory of Computing}, 2019.

\bibitem[AO07]{auer2007NewPACBound}
Peter Auer and Ronald Ortner.
\newblock A new {{PAC}} bound for intersection-closed concept classes.
\newblock {\em Machine Learning}, 2007.

\bibitem[Bar92]{bartlett1992learning}
Peter~L Bartlett.
\newblock Learning with a slowly changing distribution.
\newblock In {\em {Conference on Learning Theory~(COLT)}}, 1992.

\bibitem[BDPSS09]{ben2009agnostic}
Shai Ben-David, D{\'a}vid P{\'a}l, and Shai Shalev-Shwartz.
\newblock Agnostic online learning.
\newblock In {\em {Conference on Learning Theory~(COLT)}}, 2009.

\bibitem[{Ben}15]{ben-david20152NotesClasses}
Shai {Ben-David}.
\newblock 2 {{Notes}} on {{Classes}} with {{Vapnik-Chervonenkis Dimension}} 1.
\newblock {\em arXiv:1507.05307}, 2015.

\bibitem[BHMZ20]{bousquet2020proper}
Olivier Bousquet, Steve Hanneke, Shay Moran, and Nikita Zhivotovskiy.
\newblock Proper learning, helly number, and an optimal svm bound.
\newblock In {\em {Conference on Learning Theory~(COLT)}}, 2020.

\bibitem[Bru20]{Brunel2020Deviation}
Victor-Emmanuel Brunel.
\newblock Deviation inequalities for random polytopes in arbitrary convex bodies.
\newblock {\em Bernoulli}, 2020.

\bibitem[CAL94]{cohn1994improving}
David Cohn, Les Atlas, and Richard Ladner.
\newblock Improving generalization with active learning.
\newblock {\em Machine learning}, 1994.

\bibitem[CBJD18]{caron2018deep}
Mathilde Caron, Piotr Bojanowski, Armand Joulin, and Matthijs Douze.
\newblock Deep clustering for unsupervised learning of visual features.
\newblock In {\em {European conference on computer vision~(ECCV)}}, pages 132--149, 2018.

\bibitem[Dar15]{darnstadt2015OptimalPACBound}
Malte Darnst{\"a}dt.
\newblock The optimal {{PAC}} bound for intersection-closed concept classes.
\newblock {\em Information Processing Letters}, 2015.

\bibitem[EY{\etalchar{+}}10]{el2010foundations}
Ran El-Yaniv et~al.
\newblock On the foundations of noise-free selective classification.
\newblock {\em {Journal of Machine Learning Research}}, 2010.

\bibitem[FV13]{frenay2013classification}
Beno{\^\i}t Fr{\'e}nay and Michel Verleysen.
\newblock Classification in the presence of label noise: a survey.
\newblock {\em IEEE transactions on neural networks and learning systems}, 2013.

\bibitem[GEY17]{geifman2017selective}
Yonatan Geifman and Ran El-Yaniv.
\newblock Selective classification for deep neural networks.
\newblock {\em {Neural Information Processing Systems~(NeurIPS)}}, 2017.

\bibitem[GEY19]{gelbhart2019relationship}
Roei Gelbhart and Ran El-Yaniv.
\newblock The relationship between agnostic selective classification, active learning and the disagreement coefficient.
\newblock {\em {Journal of Machine Learning Research}}, 2019.

\bibitem[GHMS23]{goel2023AdversarialResilienceSequential}
Surbhi Goel, Steve Hanneke, Shay Moran, and Abhishek Shetty.
\newblock Adversarial {{Resilience}} in {{Sequential Prediction}} via {{Abstention}}.
\newblock {\em {Neural Information Processing Systems~(NeurIPS)}}, 2023.

\bibitem[H{\etalchar{+}}14]{hanneke2014theory}
Steve Hanneke et~al.
\newblock Theory of disagreement-based active learning.
\newblock {\em Foundations and Trends{\textregistered} in Machine Learning}, 2014.

\bibitem[Han09]{hanneke2009theoretical}
Steve Hanneke.
\newblock {\em Theoretical foundations of active learning}.
\newblock Carnegie Mellon University, 2009.

\bibitem[Han16]{hanneke2016refined}
Steve Hanneke.
\newblock Refined error bounds for several learning algorithms.
\newblock {\em {Journal of Machine Learning Research}}, 2016.

\bibitem[Hel23]{helly1923UberMengenKonvexer}
Ed~Helly.
\newblock {{\"U}ber Mengen konvexer K{\"o}rper mit gemeinschaftlichen Punkte.}
\newblock {\em Jahresbericht der Deutschen Mathematiker-Vereinigung}, 1923.

\bibitem[HLM21]{hanneke2021online}
Steve Hanneke, Roi Livni, and Shay Moran.
\newblock Online learning with simple predictors and a combinatorial characterization of minimax in 0/1 games.
\newblock In {\em {Conference on Learning Theory~(COLT)}}, 2021.

\bibitem[HLW94]{haussler1994predicting}
David Haussler, Nick Littlestone, and Manfred~K Warmuth.
\newblock Predicting $\{$0, 1$\}$-functions on randomly drawn points.
\newblock {\em Information and Computation}, 1994.

\bibitem[HM25]{hardt2025PerformativePredictionFuture}
Moritz Hardt and Celestine {Mendler-D{\"u}nner}.
\newblock Performative prediction: Past and future.
\newblock {\em arXiv:2310.16608}, 2025.

\bibitem[Hod97]{hodges1997ShorterModelTheory}
Wilfrid Hodges.
\newblock {\em A {{Shorter Model Theory}}}.
\newblock Cambridge University Press, 1997.

\bibitem[HPV{\etalchar{+}}19]{he2019device}
Junfeng He, Khoi Pham, Nachiappan Valliappan, Pingmei Xu, Chase Roberts, Dmitry Lagun, and Vidhya Navalpakkam.
\newblock On-device few-shot personalization for real-time gaze estimation.
\newblock In {\em {International conference on computer vision~(ICCV)}}, 2019.

\bibitem[HSL{\etalchar{+}}24]{hu2024visual}
Yushi Hu, Otilia Stretcu, Chun-Ta Lu, Krishnamurthy Viswanathan, Kenji Hata, Enming Luo, Ranjay Krishna, and Ariel Fuxman.
\newblock Visual program distillation: Distilling tools and programmatic reasoning into vision-language models.
\newblock In {\em {Computer Vision and Pattern Recognition~(CVPR)}}, 2024.

\bibitem[HSW90]{helmbold1990learning}
David Helmbold, Robert Sloan, and Manfred~K Warmuth.
\newblock Learning nested differences of intersection-closed concept classes.
\newblock {\em Machine Learning}, 1990.

\bibitem[HY15]{hanneke2015minimax}
Steve Hanneke and Liu Yang.
\newblock Minimax analysis of active learning.
\newblock {\em {Journal of Machine Learning Research}}, 2015.

\bibitem[KPR90]{kuh1990learning}
Anthony Kuh, Thomas Petsche, and Ronald Rivest.
\newblock Learning time-varying concepts.
\newblock In {\em {Neural Information Processing Systems~(NeurIPS)}}, 1990.

\bibitem[Lit88]{littlestone1998}
Nick Littlestone.
\newblock Learning quickly when irrelevant attributes abound: A new linear-threshold algorithm.
\newblock {\em Machine learning}, 1988.

\bibitem[LSJW{\etalchar{+}}23]{liu2023cache}
Yuezhou Liu, Lili Su, Carlee Joe-Wong, Stratis Ioannidis, Edmund Yeh, and Marie Siew.
\newblock Cache-enabled federated learning systems.
\newblock In {\em International Symposium on Theory, Algorithmic Foundations, and Protocol Design for Mobile Networks and Mobile Computing}, 2023.

\bibitem[MPZ21]{miller2021OutsideEchoChamber}
John Miller, Juan~C. Perdomo, and Tijana Zrnic.
\newblock Outside the echo chamber: Optimizing the performative risk.
\newblock In {\em {International Conference on Machine Learning~(ICML)}}, 2021.

\bibitem[MT95]{mammen1995AsymptoticalMinimaxRecovery}
E.~Mammen and A.~B. Tsybakov.
\newblock Asymptotical {{Minimax Recovery}} of {{Sets}} with {{Smooth Boundaries}}.
\newblock {\em The Annals of Statistics}, 1995.

\bibitem[PGA{\etalchar{+}}23]{park2023towards}
Jinkyung Park, Joshua Gracie, Ashwaq Alsoubai, Gianluca Stringhini, Vivek Singh, and Pamela Wisniewski.
\newblock Towards automated detection of risky images shared by youth on social media.
\newblock In {\em Companion proceedings of the ACM web conference}, 2023.

\bibitem[PZMH20]{perdomo2020PerformativePrediction}
Juan Perdomo, Tijana Zrnic, Celestine {Mendler-D{\"u}nner}, and Moritz Hardt.
\newblock Performative {{Prediction}}.
\newblock In {\em {International Conference on Machine Learning~(ICML)}}, 2020.

\bibitem[RR22]{rubinstein2022UnlabelledSampleCompression}
Joachim Rubinstein and Benjamin Rubinstein.
\newblock Unlabelled {{Sample Compression Schemes}} for {{Intersection-Closed Classes}} and {{Extremal Classes}}.
\newblock In {\em {Neural Information Processing Systems~(NeurIPS)}}, 2022.

\bibitem[RS88]{rivest1988learning}
Ronald~L Rivest and Robert Sloan.
\newblock Learning complicated concepts reliably and usefully.
\newblock In {\em {Association for the Advancement of Artificial Intelligence~(AAAI)}}, 1988.

\bibitem[RST11]{rakhlin2011OnlineLearningStochastica}
Alexander Rakhlin, Karthik Sridharan, and Ambuj Tewari.
\newblock Online learning: Stochastic and constrained adversaries.
\newblock In {\em {Neural Information Processing Systems~(NeurIPS)}}, April 2011.

\bibitem[She78]{shelah1978ClassificationTheoryNumbers}
Saharon Shelah.
\newblock {\em Classification Theory and the Numbers of Non-Isomorphic Models}.
\newblock Studies in Logic and the Foundations of Mathematics. North-Holland, rev. ed edition, 1978.

\bibitem[SSZ{\etalchar{+}}23]{shumailov2023curse}
Ilia Shumailov, Zakhar Shumaylov, Yiren Zhao, Yarin Gal, Nicolas Papernot, and Ross Anderson.
\newblock The curse of recursion: Training on generated data makes models forget.
\newblock {\em arXiv:2305.17493}, 2023.

\bibitem[SSZ{\etalchar{+}}24]{shumailov2024ai}
Ilia Shumailov, Zakhar Shumaylov, Yiren Zhao, Nicolas Papernot, Ross Anderson, and Yarin Gal.
\newblock {AI} models collapse when trained on recursively generated data.
\newblock {\em Nature}, 2024.

\bibitem[SZS{\etalchar{+}}24]{sun2024we}
Zhen Sun, Zongmin Zhang, Xinyue Shen, Ziyi Zhang, Yule Liu, Michael Backes, Yang Zhang, and Xinlei He.
\newblock Are we in the {AI}-generated text world already? quantifying and monitoring aigt on social media.
\newblock In {\em Association for Computational Linguistics}, 2024.

\bibitem[VRW23]{veselovsky2023artificial}
Veniamin Veselovsky, Manoel~Horta Ribeiro, and Robert West.
\newblock Artificial artificial artificial intelligence: Crowd workers widely use large language models for text production tasks.
\newblock {\em arXiv:2306.07899}, 2023.

\bibitem[WEY11]{wiener2011agnostic}
Yair Wiener and Ran El-Yaniv.
\newblock Agnostic selective classification.
\newblock In {\em {Neural Information Processing Systems~(NeurIPS)}}, 2011.

\bibitem[Yan25]{yan2025tilde}
Chao Yan.
\newblock An$\backslash$tilde $\{$O$\}$ ptimal differentially private learner for concept classes with vc dimension 1.
\newblock {\em arXiv preprint arXiv:2505.06581}, 2025.

\bibitem[YCA20]{YM2020Selflabelling}
Asano YM., Rupprecht C., and Vedaldi A.
\newblock Self-labelling via simultaneous clustering and representation learning.
\newblock In {\em {International Conference on Learning Representations~(ICLR)}}, 2020.

\end{thebibliography}

\newpage
\appendix

\section{Appendix: Thresholds}

\subsection{Adaptive Adversary}
\label{app:adap_threshold}

\advthresh*
\begin{proof}
We present the argument for the upper bound, and refer to the general case in~\Cref{thm:general-adver} for the lower bound.

The conservative strategy learner only makes a mistake when $\hat{y}_t=0$ but the label is $y_t=1$: The learner starts with $\hat h_1\equiv 0$ and until the first sample $x_{t_1}$ with $y_{t_1}=1$ the learner makes no mistakes and does not update its prediction. This label must come from the true hypothesis $y_{t_1}=f^*(x_{t_1})=1$ as all $\hat h_i(x_{t_1})=0$ for $i<{t_1}$. The learner updates the prediction to \(\bI\bc{\cdot \geq x_{t_1}}\). Now the learner still predicts 0 for all \(x < x_{t_1}\), and thus never makes a mistake when $y_t=0$.\\
This shows that every mistake corresponds exactly to observing a 1 at some \(x_t\) strictly below the current threshold, and as before, it holds that $y_t=f^*(x_t)=1$. Therefore, every mistake is a false negative and triggers an update. There are only $N$ possible updates and $T$ rounds, which gives $\mathcal{M}_T(\cA)\leq\min\{N,\,T\}$.
\end{proof}

\subsection{Stochastic Adversary}
\label{app:stoch_thresholds}

\stochthresh*
\begin{proof}
\subsubsection*{Upper bound}
Let $x_t\overset{\iid}{\sim}\mathcal{D}$ for $t\in[T]$ and consider $f^*=f_1 \equiv 1$ with an adversary that reveals the true label 1 every time. Indeed, by~\Cref{lem:clos_only_true}, this is a worst-case adversary for the closure algorithm.
Define $S_t=(x_1, x_2,\dots,x_t)$ as the sequence of the first \(t\) positively labelled sampled points and $Y_{t}:=\bI\bc{\,x_t < \min \{S_{t-1}\}}$. Note that the learner as described in \Cref{alg:conservative_thresh} incurs a mistake at $x_t$ exactly when $Y_t=1$.
Therefore, we get an upper bound of $\mathcal{M}_T(\cA)=\sum_{t=1}^TY_t\leq \min\{N,T\}$.\\
For $T<N$, we can improve this using an exchangeablilty argument to $\bE\bs{\mathcal{M}_T(\cA)}=\bE\bs{\sum_{t=1}^TY_t}\leq\sum_{t=1}^T\frac{1}{t}\leq\log T+\gamma$ where $\gamma$ is a constant term. This holds because $x_t$ are sampled i.i.d.\ and therefore 
\begin{equation*}
    P(x_1<\min S_t\setminus\{x_1\})=\dots=P(x_t<\min S_t\setminus\{x_t\})\leq \frac{1}{t}.
\end{equation*}
This proves an upper bound of $O\bigl(\min\{N,\log T\}\bigr)$.
\subsubsection*{Lower bound}
We start with the case \(\log T > 2N\).
Consider the following distribution \(\cD\) over \([N]\): for $k\in[N-1]$ set \(\Pr_{X \sim \cD} \br{X = k} = \frac{1}{2} \cdot \frac{1}{3^{N - k}}\) and let $P(X=N)=1-\sum_{k=1}^{N-1}P(X=1)\geq \frac{3}{4}$. Observe that for any \(k \in [N-1]\)
\begin{equation}
\label{eq:prob_ratio}
    \frac{\Pr\br{X = k}}{\Pr\br{X \leq k}} = \frac{3^{k - N}}{\sum_{i = 1}^k 3^{i - N}} = \frac{2\cdot 3^k}{3 \cdot (3^k - 1)} \geq \frac{2}{3}.
\end{equation}

Let \(x_1, \ldots, x_T \overset{\iid}{\sim} \cD\) be the sampled input points. 
Let \(t_k \coloneqq \min\bc{t \in [T/2] \text{ such that } x_t \leq k}\), with \(t_k = \infty\) when \(x_t > k\) for all \(t \in [T/2]\). By construction it holds that $1=t_N\leq t_{N-1}\leq \dots \leq t_2\leq t_1\leq \infty$.
Note that for all \(k \in [N]\),
\begin{equation}
\begin{aligned}
\label{eq:tk_finite}
    \Pr(t_k < \infty) &\geq \Pr(t_1 < \infty) = 1 - \Pr(\text{for all } t \in [T/2],\, x_t > 1)\\
    &= 1 - \br{1 - \frac{3}{2 \cdot 3^N}}^{T/2} \geq 1 - \exp\br{-\frac{3 T}{4 \cdot 3^N}} \geq 1 - \exp(-3/4) \geq 1/2.
\end{aligned}
\end{equation}
We define \(A_k \coloneqq \bc{t_k < \infty \text{ and }x_{t_k} = k}\) and compute using~\cref{eq:prob_ratio} for $k\in[N-1]$
\begin{equation}
\label{eq:pr_ak}
    \Pr(A_k) = \Pr(x_{t_k} = k \mid t_k < \infty) \Pr(t_k < \infty) \geq \frac{1}{2}\Pr(X = k \mid X \leq k) \geq \frac{1}{3}.
\end{equation}
Note that for $k=N$ we have $t_N=1$ and therefore $\Pr(A_N)=\Pr(x_1=N)\geq1/2$ by construction of the the distribution $\mathcal{D}$. \\
Our constructed adversary chooses label \(1\) for every point \(x_t\), for \(t \in [T/2]\). Let \(\hat h_1, \ldots, \hat h_{T/2}\) be the hypotheses output by the learner.
If for all \(t \in [T/2]\), we have \(\Trp_t = \emptyset\), then if \(f^\ast \equiv 1\), the learner made a mistake each time it received \(x_t < \min\bc{x_1, \ldots, x_{t-1}}\).
Equivalently, an error happens on some $k\in[N]$ if and only if the first time such a $x_t \leq k$ is sampled, it is exactly $k$, which is the event $A_k$. The remaining \(T/2\) samples from \(\cD\) can be labelled arbitrarily. Therefore 
\begin{equation}
    \bE \bs{\cM_T(\cA)}\geq \bE\bs{\sum_{t=1}^{T/2}\mathbb{I}\bc{x_t<\min\{x_1,\dots,x_{t-1}\}}}=\bE\bs{\sum_{k=1}^N \mathbb{I}\bc{A_k}}=\sum_{k=1}^N \Pr(A_k) \geq \frac{N}{3}.
\end{equation}

If for some \(t^\ast \in [T/2]\), we have \(\Trp_{t^\ast} \neq \emptyset\),~\Cref{thm:no_trap} implies that there exists an adversary forcing the learner to make linear in \(T \Pr_{X \sim \cD}\br{X = 1}\) mistakes. Since \(\log T > 2N\), we get that \(T \Pr_{X \sim \cD}\br{X = 1} \geq T 3^{-N} \geq N\).

In the remaining case we have \(\log T \leq 2N\). Here,~\cref{eq:tk_finite} does not necessarily hold for all \(k \in [N]\), but a straightforward computation shows that for \(k \geq N - \frac{\log T}{2} + O(1)\), we still have that \(\Pr\br{t_k < \infty} \geq 1/2\). Thus, \(\Pr(A_k) \geq 1/3\) for \(k \geq N - \frac{\log T}{2} + O(1)\) implying the claimed \(\Omega(\log T)\) lower bound.

\end{proof}

\section{Appendix: General Setting}

\subsection{Extended Threshold Dimension}

Before we introduce the Extended Threshold dimension, we give some intuition for the Threshold dimension. Namely, we show that for the intersection-closed classes $\HC$, the longest chain in $\HC$ is of length $\tdim{\HC}$. 
\label{sec:ext-thresh-app}
\begin{defn}
    We define the depth of a hypothesis class $\HC$ as the longest chain in $\HC$: $\depth{\HC}=\max\bc{L\geq 0 : \exists h_0\supsetneq h_1\supsetneq\dots\supsetneq h_L\in\HC}.$ 
\end{defn}

\begin{restatable}{lem}{threshdepth}
\label{lem:tdimeqdepth}
Let $\HC\subseteq2^{\X}$ be intersection–closed. Then, $\depth{\HC}=\tdim{\HC}$. 
\end{restatable}
\begin{proof}[Proof of~\Cref{lem:tdimeqdepth}]
(\(\depthop\Rightarrow\tdimop\))  
Take a longest strict chain $h_0\supsetneq\dots\supsetneq h_L$ in
$\HC$.  Choose $x_t\in h_t\setminus h_{t-1}$ for $t=1,\dots ,L$.  For
every $t$, the set $h_t$ intersects $\bc{x_1,\dots ,x_L}$ in exactly
the first $t$ points, hence $\bc{x_1,\dots ,x_L}$ and \(\bc{h_0, \ldots, h_L}\) are the witness sets of
length $L$.  Thus, $\tdim{\HC}\geq\depth{\HC}$.

\noindent(\(\tdimop\Rightarrow\depthop\)) By definition of threshold
dimension, there are points \(x_1<\dots<x_L\in\cX\) and hypotheses  
\( \tilde h_0, \tilde h_1,\ldots, \tilde h_L\in\HC\) such that
$\tilde h_t\cap\{x_1,\dots ,x_L\}=\{x_1,\dots ,x_t\}$ for every $t=0,\dots
,L$.  Define \(h_t \coloneqq \bigcap_{j \geq t} \tilde h_j \in \HC\). Since $x_{t+1}\notin \tilde h_t$, while $x_{t+1}\in \tilde h_{j}$ for all \(j \geq t+1\), each
inclusion $h_t\subsetneq h_{t+1}$ is strict. Hence $\br{h_0,\dots
,h_L}$ is a chain of length $L=\tdim{\HC}$ that lies entirely in
$\HC$. Thus,~\(\depth{\HC}\geq L=\tdim{\HC}.\)
\end{proof}

\begin{restatable}{lem}{threshhccupper}\label{lem:threshdim-upper}
    Let $\HC\subseteq2^\X$ a hypothesis class and $\HCC$ the closure of this class. Then $\tdim{\HCC}\leq\abs{\HC}$.
\end{restatable}  
\begin{proof}
    Recall the definition of $\HCC=\{\bigcap_{h\in S}h|S\subseteq \HC\}$ and let $|\HC|=k<\infty$.
    Let \(\bar h_0 \supsetneq \bar h_1 \supsetneq \ldots \supsetneq \bar h_{\ell}\) be an arbitrary chain in \(\HCC\).
    Note that for each \(0 \leq t \leq \ell\) we have \(\bar h_t = \bigcap_{j \leq t} \bar h_j\). 
    Also, for each \(1 \leq t \leq \ell\), there exists \(x_t\), such that \(\bar h_t(x_t) = 1 \neq 0 = \bar h_{t+1}(x_t)\). 
    Therefore, all \(x_t\)'s are distinct and we obtain \(\ell \leq \abs{\HC} = k\). The claim follows from~\Cref{lem:tdimeqdepth}.
\end{proof}

\exthreshbounds*
\begin{proof}[Proof of \Cref{prop:exthresh_bounds}]
    For the first inequality, consider \(x_1, \ldots, x_k\) and \(h_0, h_1, \ldots, h_k\) be the witness sets for \(\cH\). Fix \(f \subseteq \cX\) and let \(S \coloneqq \bc{x_1, \ldots, x_k} \setminus f\). Without loss of generality assume that \(s \coloneqq \abs{S} \geq k / 2\). This is true because of the observation that $\tdim{\HC^f}=\tdim{\HC^{(1-f)}}$ where $(1-f)\coloneqq\bc{x\in\X\mid f(x)=0}$ because of symmetries. Let \(x_{i_1}, \ldots, x_{i_s}\) be the elements of the set \(S\) with \(i_1 < i_2 < \ldots < i_s\). Then, we claim that \(x_{i_1}, \ldots, x_{i_s}\) and \(h_0, h_{i_1}, \ldots, h_{i_s}\) is a valid witness set (possibly not of the largest cardinality) for \(\cH^f\).
    Indeed, this is a valid witness set for \(\cH\) by construction, and since \(S \cap f = \emptyset\), it remains valid for \(\cH^f\). Therefore, \(\tdim{\overline{\cH^f}} \geq \tdim{\cH^f} \geq k / 2 = \tdim{\cH} / 2\). Since the inequality holds for any \(f \subseteq \cX\), we have 
    \begin{equation*}
        \extdim{\cH} = \min_{f \subseteq \cX} \tdim{\overline{\cH^f}} \geq \tdim{\cH} / 2. 
    \end{equation*}
    The second inequality follows by definition of $\extdim{\HC}$ and the last inequality follows from \Cref{lem:threshdim-upper}.\\
    The statement for intersection-closed classes follows immediately as \(\overline{\cH} = \cH\)  for intersection-closed \(\cH\) and see \Cref{prop:exthresh_blowup} for a construction of $\HC^\ast$.
\end{proof}

While for intersection-closed classes we have \(\extdim{\cH} = \Theta(\tdim{\cH})\), the former quantity can be much larger for certain non intersection-closed hypothesis classes, as the next result shows.

\begin{prop}\label{prop:exthresh_blowup}
    For any \(N \geq 2\), there exists a hypothesis class \(\cH^\ast\) over the domain \(\cX = [2N]\), such that \(\tdim{\cH^\ast} = 3\), but \(\extdim{\cH^\ast} \geq N\).
\end{prop}

\begin{proof}
    Let \(\cH_s^N \coloneqq \bc{\bI\bc{ \cdot = i} \text{ for } i \in [N]} \cup f_0 \subseteq 2^{[N]}\) be the class of singletons (also known as the \(\mathrm{Point}\) class) and \(\cH_{rs}^N \coloneqq \bc{\bI\bc{ \cdot \neq i} \text{ for } i \in [N]} \cup f_1 \subseteq 2^{[N]}\) be the class of reverse singletons, where \(f_0 \equiv 0\) and \(f_1 \equiv 1\). Let \(\cH = \cH_s^{2N} \cup \cH_{rs}^{2N}\). Note that for any \(h \in \cH\), we have \[\min \br{\abs{\bc{x \in \cX, \text{ s.t. } h(x) = 0}}, \abs{\bc{x \in \cX, \text{ s.t. } h(x) = 1}}} \leq 1.\] Assume that \(\tdim{\cH} \geq 4\) and let \(x_1, \ldots, x_4\) and \(h_0, h_1, \ldots, h_4\) be the witness sets. We arrive at contradiction, since neither of the elements of \(\cH\) could be \(h_2\), which has \[\min\br{\abs{\bc{x \in \cX, \text{ s.t. } h_2(x) = 0}}, \abs{\bc{x \in \cX, \text{ s.t. } h_2(x) = 1}}} = 2.\] Thus, \(\tdim{\cH} \leq 3\), and it is easy to construct witness sets of cardinality \(3\).

    For the second part, fix \(f \subseteq \cX\). From the symmetries of \(\cH\), we can assume without loss of generality that \(f = [k]\) for \(k \leq N\). Then, if we restrict \(\cH\) to \([N + 1, \ldots, 2N]\), we observe that it contains a copy of \(\cH_{rs}^N\). The claim follows, since \(\tdim{\overline{\cH_{rs}^N}} \geq N \).
\end{proof}

Lastly, we provide a result showing that, in some cases, finding a good $f$-representation leads to a significant improvement over simply running~\Cref{alg:closure_explicit} with $\HC$. We show that there exists an $f$-representation for which $\HC^f$ is intersection-closed and the Extended Threshold dimension is bounded, while $\tdim{\HCC}$ can be arbitrarily large. 
\begin{prop}
    For any $N\geq 2$, there exists a hypothesis class $\HC$ over $\X=[N]$, such that $\tdim{\HCC}=N$ while $\extdim{\HC}\leq2$. 
\end{prop}

\begin{proof}
    Consider the class of reverse singletons \(\cH_{rs}^N \coloneqq \bc{\bI\bc{ \cdot \neq i} \text{ for } i \in [N]}\cup f_1 \subseteq 2^{[N]}\) where $f_1$ is the all-ones function on $\X=[N]$. This class is union-closed and there exists an $f$-representation for $f=1$ such that $\HC^f$ is intersection-closed. In this case, $\extdim{\HC}\leq\tdim{\overline{\HC^f}}=2$, where the inequality holds by definition of the Extended Threshold dimension, whereas $\tdim{\HCC}=N$. 
\end{proof}

\subsection{$f$-representations for classes with VC-dimension 1}\label{app:f_for_vcd1}

We restate the result from \cite{ben-david20152NotesClasses} that characterizes classes with VC~dimension~1:

\begin{defn}
Let $(\X,\preceq)$ be a partial order. We call $h\in\HC$ an initial segment w.r.t.\ $\preceq$ if
\begin{itemize}
  \item For all $x,y\in\X$, if $x\preceq y$ and $h(y)=1$, then $h(x)=1$; 
  \item If $h(y)=h(z)=1$ for some $y,z\in\X$, then either $y\preceq z$ or $z\preceq y$. 
\end{itemize}
Equivalently, $h\in \{\emptyset\}\cup\{I_x:x\in\X\}$ where $I_x:=\{y\in\X:y\preceq x\}$.
\end{defn}
\noindent
Reading the proof of Theorem 4 in \cite{ben-david20152NotesClasses} carefully, we see the following holds true:
\begin{thm}[\cite{ben-david20152NotesClasses}, Theorem 4]\label{thm:ben-david_notes}
Let $\cH\subseteq 2^{\cX}$ be a hypothesis class over some domain $\X$. The following are equivalent:
\begin{enumerate}
    \item $\vcd{\HC}\leq 1$,
    \item All representations $f\in\HC$ induce a tree ordering over $\X$ such that every element of $\HC^f$ is an initial segment under that ordering relation.
\end{enumerate}
\end{thm}
The next lemma summarizes invariance under certain $f$-representations for classes with VC dimension 1.
\begin{lem}\label{lem:vc1_class}
    For a hypothesis class $\HC$ and any $f$-representation $\HC^f$, it holds \(\vcd{\HC}=\vcd{\HC^f}\).\\ %
    Furthermore, when $\vcd{\HC}=1$, for any $f \in \HC$ one has 
    \[\vcd{\HC^f}=\vcd{\overline{\HC^f}}=1 \quad \text{and} \quad\tdim{\HC^f}=\tdim{\overline{\HC^f}}.\]
\end{lem}
\begin{proof}
    For the first statement for general $f$, the VC dimension result was shown in \cite{yan2025tilde} and mentioned in \cite{ben-david20152NotesClasses}\\
    For the second part, we first note that the intersection of two initial segments is again an initial segment:
    Let $h_1,h_2\in\HC^f$ and \(x_1, x_2 \in h_1 \cap h_2 \eqqcolon h_3\). 
    By~\Cref{thm:ben-david_notes}, \(h_1\) and \(h_2\) are initial segments, we can assume w.l.o.g. that \(x_1 \preceq x_2\). 
    If $h_3(x_2)=1$, we have by construction $h_1(x_2)=h_2(x_2)=1$, which implies \(h_1(x_1) = h_2(x_1) = 1\) and \(h_3(x_1) = 1\). Thus, $h_3$ is an initial segment. 
    
    Next, $\vcd{\HC^f}=\vcd{\overline{\HC^f}}$ follows directly for $f\in \HC$ as in \Cref{thm:ben-david_notes}, because $\overline{\HC^f}$ consists only of initial segments and thus has VC dimension 1. 
    For the threshold dimension, $\tdim{\HC^f}\leq\tdim{\overline{\HC^f}}$ follows immediately as $\HC^f\subseteq \overline{\HC^f}$. We show $\tdim{\overline{\HC^f}}\leq\tdim{\HC^f}$: let $\tdim{\overline{\HC^f}}=k$ and $x_1,\dots, x_k$ with $\bar{h}_0,\bar{h}_1,\dots,\bar{h}_k$ be the witness sets for $\overline{\HC^f}$. For \(j \in [k]\), let \(I_j \subseteq \HC^f\) be such that $\bar{h}_j=\bigcap_{h_i\in I_j} h_i$, and note that $\bar{h}_j$ is an initial segment. 
    Since \(\bar h_k(x_j) = 1\) for all \(0 \leq j \leq k\), we have that for all \(i, j\), either \(x_i \preceq x_j\) or \(x_j \preceq x_i\). Next, for all \(j \in [k - 1]\), it must hold that \(x_{j} \prec x_{j+1}\), otherwise \(\bar h_j\) would violate the property of the initial segment.
    
    We show that for each \(j \in [k]\), there exists $h^\ast_{j}\in I_j$, such that $x_1,\dots, x_k$ and $h^\ast_{0},h^\ast_{1},\dots,h^\ast_{k}$ are the valid witness sets for $\HC^f$, i.e., for all \(\ell \in [k]\), $\bI\bc{j\geq \ell}=\bar{h}_j(x_\ell)=h^\ast_{j}(x_\ell)$. 

    Since \(\bar{h}_j(x_{j+1}) = 0\), we can set \(h^*_j\) to a hypothesis \(h^*_j \in I_j\), such that \(h^*_j(x_{j+1}) = 0\).
    For such choice of $h^\ast_{j}$ we have $\bar{h}_j(x_\ell)=h^\ast_{j}(x_\ell)$ for all $\ell\in[k]$, as
    \begin{itemize}
        \item for $\ell\leq j$: $\bar{h}_j(x_\ell)=1$ and thus $h^\ast_j(x_\ell)=1$,
        \item for $\ell>j$: as $h^\ast_{j}$ is an initial segment and $h^\ast_{j}(x_{j+1})=0$, we have that $h_{j^*}(x_{\ell})=0$. 
    \end{itemize}
    This finishes the proof, as we showed that $\tdim{\HC^f}\geq k = \tdim{\overline{\HC^f}}$.
\end{proof}

    We showed in \Cref{prop:exthresh_blowup} that taking the closure of a class can make the threshold dimension arbitrarily large. The same holds for VC dimension: let $\X=[N]$ and consider the class of reverse singletons $\cH_{rs}^N \;:=\; \bigl\{\, h_i(x)=\bI\bc{x\neq i}\ :\ i\in[N] \bigr\} \,\cup\, \bc{f_1}, $ where $f_1\equiv 1$. Clearly $\vcd{\cH_{rs}^N}=1$. However, its closure includes all binary functions over $\X$ and thus $\vcd{\overline{\cH_{rs}^N}}=N$.

\Cref{lem:vc1_class} gives an efficient way to find an $f$-representation for classes of VC dimension 1, such that the closure algorithm is optimal up to constants. This follows directly from \Cref{thm:general-adver,thm:general-stoch}.

\subsection{Trap Region}
\label{sec:trap-region-app}
Next lemma shows that if the learner admits a non-empty \Gls{trap} at any time $t\in[T/2]$, then an adaptive or stochastic adversary can make the (expected) number of mistakes scale linearly in $T$.

\begin{restatable}{lem}{nozerooneregion}\label{thm:no_trap}
    Consider a learner \(\cA\) in the replay setting. If there exists \(t^\ast \in [T/2]\) such that \(\Trp_{t^\ast} \neq \emptyset\), then:
    \begin{enumerate}
        \item \textbf{Adaptive adversary}: The replay adversary in the adaptive setting can pick samples at steps \(t = t^\ast + 1, \ldots, T\), such that the number of mistakes of \(\cA\) is linear in $T$.
        \item \textbf{Stochastic adversary}: The replay adversary in the stochastic setting can make \(\cA\) incur an expected number of mistakes linear in $\arg \max_{x^\ast \in \Trp_{t^\ast}} T \mu(x^\ast)$, where \(\mu(\cdot)\) is the measure of the distribution \(\cD\). %
    \end{enumerate}
\end{restatable}

\begin{proof}[Proof of \Cref{thm:no_trap}]
     \noindent\textbf{Adaptive adversary}: Pick an \(x^\ast \in \Trp_{t^\ast}\) and let $h, h'\in \HC_{t^\ast-1}$ and $f,f'\in\VS_{t^\ast}$ be as in~\Cref{defn:trapt} for such \(x^\ast\). The adversary now continues to play $x^*$ for $t'>t$ with the opposite label as the learner predicted $y_{t'}=1-h_{t'}(x^\ast)$. This is possible, as both labels can be produced through replay of an old hypothesis, since $\{0,1\}=\{h(x^\ast),h'(x^\ast)\}$. Furthermore, both labels could also come from a hypothesis $f$ or $f'$ as $\{0,1\}=\{f(x^\ast),f'(x^\ast)\}$ and $f,f'\in\VS_{t'}$ for all $t'>t^\ast$. The last statement holds because  $f,f'\in\VS_{t^\ast}$ by assumption and $\VS_{t'}=\VS_{t^\ast}$ for all $t'\geq t^\ast$. 
     
     Indeed, the adversary that continues playing $x^*$ does not modify $I_{t'}$ (recall~\Cref{defn:version_space}) because there exists an $h_k\in\{h,h'\}\in\HC_{t^\ast-1}$ such that  $h_k(x_{t'})= y_{t'}$ and therefore $I_{t'}=I_{t^\ast}$.
    In the end, the adversary sets $f^*\in\{f,f'\}$ depending on which incurs more mistakes. By the pigeon-hole principle, the learner can be forced to make at least $\frac{T}{4}=\Omega(T)$ mistakes.

    \noindent\textbf{Stochastic Adversary}: Pick \(x^\ast \in \arg \max_{x \in \Trp_{t^\ast}} \mu(x)\) and let $h, h'\in \HC_{t^\ast-1}$ and $f,f'\in\VS_{t^\ast}$ again be as in~\Cref{defn:trapt}. Let the samples $x_{t^\ast+1},\dots,x_T\overset{\iid}{\sim}\cD$.
    For any $t'>t^\ast$, the adversary chooses to give the opposite label $y_{t'}=1-h_{t'}(x_{t'})$ whenever $x_{t'}=x^*$ and the replayed label $y_{t'}=h_{t^\ast}(x_{t'})$ for $x_{t'}\neq x^*$. 
    
    By the same arguments as before, this labeling is possible as both labels 0 and 1 could be replayed or come from some true $f^*\in\{f,f'\}$. Note that we also have $I_{t'}=I_{t^\ast}$ as there exists $ h_k\in\HC_{t^\ast}$, such that $y_{t'}=h_k(x_{t'})$: for $x_{t'}=x^*$ choose such a $h_k\in\{h,h'\}$  and for $x_{t'}\neq x^*$ choose $h_k=h_{t^\ast}$. At the end, the adversary chooses such an $f^\ast\in\{f,f'\}$ which incurs larger number of mistakes, which by the pigeon-hole principle is at least half of the times $x_t=x^*$ is sampled. This gives the following bounds 
    \begin{align*}
        \bE\bs{\mathcal{M}(T)}&\geq\bE\sum_{t'=t^\ast+1}^T\bI\bc{x_{t'}=x^* \text{ and } y_{t'}\neq h_{t'}(x_{t'})  \text{ and } y_{t'}=f^*(x_{t'})}\\
        &=\sum_{t' > t^\ast:\ y_{t'} = f^\ast(x_{t'})} \Pr(x_{t'} = x^\ast)
        \geq\frac{T}{4}\Pr_{x \sim \cD}(x=x^*)=\Omega(T\mu(x^*)),
    \end{align*}
    since 
    the adversary always plays the opposite label as predicted by the learner.
\end{proof}

\begin{example}\label{rem:no_ambiguity-subset} 
    For a learner that starts with predicting $h_{\min}\equiv 0$, the set of points $I_t$ (see~\Cref{defn:version_space}) is a subset of those samples with label 1: $I_t\subseteq \{i\in\{1,\dots,T\}\mid y_i=1\}$. In this case $\VS_t=\{h'\in\HC \mid \forall i\in I_t: h'(x_i)=1\}=\{h'\in \HC\mid \clos_\HC((x_i)_{i\in I_t})\subseteq h'\}$ and \Cref{defn:trapt} simplifies to 
    \begin{equation}
    \Trp_t:=\bc{x\in\X \setminus \clos_\HC((x_i)_{i\in I_t}) \mid \exists h,h' \in \HC_{t-1} \text{ s.t. } h(x)\neq h'(x)}.
    \end{equation}
\end{example}

\begin{restatable}{corollary}{hypochain}\label{cor:hypothesis_chain}
If a learner starts with predicting $h_{\min}$ and at time $t\in[T]$ outputs a hypothesis $h_t$ with $h_t\not\subseteq \clos_\HC((x_i)_{i\in I_{t-1}})$, then \(\Trp_t \neq \emptyset\) and they can be forced to make linear mistakes by \Cref{thm:no_trap}. Furthermore, a learner that outputs $h_t\subsetneq \clos_\HC((x_i)_{i\in I_{t-1}})$ incurs the same or more mistakes as a learner that outputs $h_t=\clos_\HC((x_i)_{i\in I_{t-1}})$.
\end{restatable}

\begin{proof}
Let $h=h_0$ and $h'=h_t$ with $f\in\VS_{t}$ such that $\clos_{\HC}((x_i)_{i\in I_{t-1}})\subseteq f$. Note that by assumption there exists an $x\in\X:h_t(x)=1\neq 0=f(x)$. Set $f'\in\VS_{t}$ such that $\clos_{\HC}((x_i)_{i\in I_t}\cup \{x\})\subseteq f'$ and the conditions for a \Gls{trap} as in \Cref{defn:trapt} with \Cref{rem:no_ambiguity-subset} are satisfied. \\
For the second claim, note that for all $x\notin f$ and $x\in h_t$: $h_t(x)=f(x)$ agree, so the learner incurs no mistakes. However, $\forall x\in f\setminus h_t: h_t(x)=0\neq1=f(x)$ which is a true-label mistake as $f\in\VS_t$.
\end{proof}

\subsection{Closure Algorithm}
\label{sec:clos-alg-app}
\begin{restatable}{lem}{closureconserv}\label{lem:closure-conserv}
Let \(\HCC\) be the intersection closure of a hypothesis class \(\HC\). Let \(S_t = \{x_1, \dots, x_t\}\subset\cX\) be any set of points with label \(1\), and let $h_t = \clos_{\HC}\br{S_t}$~(Line~\ref{step:hyp-close} of~\Cref{alg:closure_explicit}). If there exists \(h' \in \HC\) consistent with \(S_t\), then for any point $x$ where \(h'(x)=0\), it must be that \(h_t(x)=0\).
\end{restatable}

\begin{proof}[Proof of~\Cref{lem:closure-conserv}]
    We prove this by contradiction. Assume there exists a point \(x \in \cX\) and a hypothesis \(h' \in \HC\) such that \(S_t \subseteq h'\)~(\(h'\) is consistent with \(S_t\)), \(h'(x) = 0\) and \(h_t(x) = 1\). By definition of the intersection closure, $h' \in \HC$ implies $h' \in \HCC$. The hypothesis $h_t$ is defined as \(\clos_{\HC}(S_t)\), which is the smallest set in $\HCC$ that contains $S_t$. Since $h' \in \HCC$ and \(h'\) also contains \(S_t\), then it must be the case that $h_t \subseteq h'$.
    
    Our assumption that \(h_t(x)=1\) means \(x \in h_t\). Since \(h_t \subseteq h'\), we have that \(x \in h'\). This contradicts our assumption that \(h'(x)=0\). Thus, the initial assumption must be false.
\end{proof}

\begin{restatable}{lem}{onlyposerr}\label{lem:clos_only_true}
    The closure algorithm,~\Cref{alg:closure_explicit}, only incurs a mistake if $y_t=1$ and $\hat{y}_t=0$. In particular, it is never the case that $y_t=0$ and $\hat{y}_t=1$. 
\end{restatable}
\begin{proof}
     We show that for the true hypothesis $h^*\in\HC$ the update step in algorithm \Cref{alg:closure_explicit} ensures that the prediction $h_t\in\HCC$ satisfies $h_t\subseteq h^*$ for all $t\in\{1,\dots,T+1\}$. This proves the claim. \\
     First note that $\HC\subseteq\HCC$ and thus $h^*\in\HCC$. 
     The claim follows by induction, as initially it holds $h_1=\bigcap_{h\in\HCC}h\subseteq h^*$ because $h^*\in\HCC$.
     Now assume $h_{t}\subseteq h^*$. If we do not update then $h_{t+1}=h_{t}\subseteq h^*$ holds trivially.
     Note that we only trigger the update step $h_{t+1}=\clos_{\HC}(h_{t}\cup \{x_t\})$ if $y_t=1$ and $x_t\notin h_{t}$ (i.e. $\hat{y}_t=h_t(x_t)=0$). Therefore, no old hypothesis could have been used by the adversary for $y_t$, as all old hypothesis predicted 0 on this sample. %
     This implies that $y_t=h^*(x_t)=1$. 
     Thus $h_{t}\subseteq h^*$ and $\{x_t\}\in h^*$ ensures $h_{t}\cup\{x_t\}\subseteq h^*$ and thus $h^*\in\HC$ is consistent.     %
     By \Cref{lem:closure-conserv} the learner that predicts with $h_{t+1}=\clos_{\HC}(h_{t}\cup \{x_t\})$ thus also ensures $h_{t+1}\subseteq h^*$.
\end{proof}

\advgeneral*

\begin{proof}

\noindent\textbf{Upper bound}
    Pick \(f \in \arg \min_{f' \subseteq \cX} \tdim{\overline{\cH^{f'}}}\).
    As shown in~\Cref{rem:f-rep}, we can consider the case of learning where $\HC$ is given already as $\HC^f$, taking the inverse operation then immediately gives $\HC$. Thus, without loss of generality \(f \equiv 0\), which implies \(\cH^f = \cH\). The closure algorithm,~\Cref{alg:closure_explicit}, starts with the hypothesis $\hat h_0  = \cap_{h \in \cH} h \in \HCC$ and upon each mistake on an example $x_t$, it updates its prediction to $h_{t+1}=\clos_{\HC}(h_{t}\cup\{x_t\}) \in \HCC$. Let \(\hat x_1, \ldots, \hat x_{\tau}\) be the samples on which the closure algorithm made a mistake and for each \(k \in [\tau]\) let \(\hat h_k\) be the updated hypothesis after the \(k\)-th mistake. By the construction of the closure algorithm, we have that \(\hat h_0  \subsetneq \hat h_1 \subsetneq \hat h_2 \subsetneq \ldots \subsetneq \hat h_{\tau}\), and furthermore, \(\hat h_i(x_j) = \bI\bc{j \leq i}\). Thus, by the definition of the Threshold dimension, we have that \(\tau \leq \tdim{\overline{\cH^{f}}}\). 
    In general, since the chosen representation \(f\) attains the smallest possible value, we get \(\tau \leq \extdim{\cH}\).

\noindent\textbf{Lower bound} 
    For the lower bound, let \(f = \hat h_1\) be the first output of \(\cA'\). In the following, we will show that it is possible to force \(\cA'\) to make \(\Omega\br{\tdim{\overline{\cH^{f}}}}\) mistakes.
    Again, without loss of generality (see~\Cref{rem:f-rep}), assume that \(f \equiv 0\), so that \(\cH^f = \cH\).

    Let \(x_1, \ldots, x_k\) and \(h_0 \subsetneq h_1 \subsetneq \ldots \subsetneq h_k\) be the witness sets for \(\overline{\cH}\) with \(h_k = \bigcap_{h \in K} h\) for some \(K \subseteq \HC\).

    For the steps \(t = 1, \ldots, k\), the adversary queries \(\cA'\) with points \((x_t, 1)\). Let \(\hat h_t\) be the hypothesis that \(\cA'\) uses at step \(t\). If for all \(t, s \in [k]\), it holds that \(\hat h_{t}(x_s) = \bI\bc{s < t}\), then we can pick any \(f^\ast \in K\) to obtain \(\cM_T(\cA') \geq k\).
    
    Otherwise, there exist \(t \in [k]\) and \(s \geq t\), such that \(\hat h_t(x_s) = 1\). Let \(t^\ast\) correspond to the smallest such \(t \in [k]\), and \(s^\ast\) to the corresponding \(s \in [k]\). 
    By the definition of \(\HCC\) and since \(\hat h_{s^\ast - 1} \subsetneq h_{s^\ast}\), there exist two sets \(A \subsetneq B \subseteq \cH\), such that \(h_{s^\ast - 1} = \bigcap_{h \in B} h\) and \(h_{s^\ast} = \bigcap_{h \in A} h\). Pick \(f_1\) arbitrarily from \(A\) and \(f_2\) from \(B \setminus A\) and note that 
    \begin{equation}
        1 = f_1(x_{s^\ast}) \neq f_2(x_{s^\ast}) = 0 \qquad \text{ and } \qquad f_1(x_s) = f_2(x_s) = 1 \quad \text{ for all } s \in [s^\ast - 1].
    \end{equation}
    Thus, \(f_1, f_2 \in \VS_{s^\ast}\) and we showed that 
    \[
    \bc{f_1(x_{s^\ast}), f_2(x_{s^\ast})} =  \bc{\hat h_0(x_{s^\ast}), \hat h_t(x_{s^\ast})} = \bc{0, 1},
    \]
    which implies that \(x_{s^\ast} \in \Trp_{t^\ast}\). The claim then follows from~\Cref{thm:no_trap}.    
    Since the lower bound holds for any hypothesis \(f\) first output by \(\cA'\), we obtain that for all learners \(\cA\), we have \(\cM_T(\cA) = \Omega(\min_{f \subseteq \cX} \tdim{\overline{\cH^f}}) = \Omega(\extdim{\cH})\). 
\end{proof}

\subsection{Stochastic Adversary}
\label{app:stoch-general}

\begin{restatable}[General Upper Bound - Stochastic Adversary]{lem}{ubstochgeneral}\label{thm:general-upper-stoch}
Let $\HC\subseteq2^{\X}$ be intersection-closed with VC dimension
$d_{\mathrm{vc}}$ and $T\in\N^+$ be the time horizon.
For any $T$ there exists a learner $\cA$ (the
closure algorithm,~\Cref{alg:closure_explicit}) whose expected mistakes satisfy $\bE\bs{\cM_T(\cA)} = O\bigl(\min\br{\extdim{\HC},d_{\mathrm{vc}}\log T}\bigr).$\\ 
\end{restatable}

\begin{proof}
    The upper bound contains two terms: a deterministic term depending
    on the Threshold dimension, and a stochastic term depending
    on the VC dimension. From the properties of the closure algorithm on intersection-closed classes, it is enough to consider the case~\(f^\ast \equiv 1\) (see~\Cref{lem:clos_only_true}). And furthermore, as noted in \Cref{rem:f-rep}, we can consider learning $\HC^f$ from the beginning thus throughout the proof w.l.o.g. consider $\HC$ and ignore the flipping in the closure algorithm. \\

\noindent\emph{(i) Deterministic term:} The closure algorithm starts
at the minimal element $h_{\min}\in\HC$ and, each time it errs on a
positive example, replaces its current hypothesis $\hat h_t$ by the strict
superset $\hat h_{t+1}=\clos_{\HC'}(h_t\cup\bc{x})$. According to the argument in~\Cref{thm:general-adver}, this can be upper bounded in general by $\tdim{\HCC}=\tdim{\HC}$ as $\HC$ is intersection closed.  %
Thus,
\begin{equation}\label{eq:det-ub}
    \cM_T(\cA)\le\tdim{\HC}.
\end{equation}

\noindent\emph{(ii) Stochastic term:} Fix a round \(t\in\{1,\dots ,T\}\) and consider samples $x_i\overset{\iid}{\sim}\cD$. Set  
\[
      h_{t-1}=\clos_\HC\br{\bc{x_1,\dots ,x_{t-1}}}
\]
be the hypothesis produced by the closure algorithm before seeing
\(x_t\).   Because each earlier example is labelled \(1\), the empirical error of \(h_{t-1}\) on \(\bc{x_1,\dots ,x_{t-1}}\) is
zero. A mistake at round \(t\) occurs exactly when \(x_t\) falls
outside \(h_{t-1}\), see \Cref{lem:clos_only_true}.  Define the mistake set for a given hypothesis $h_{t-1}$
\[
      E_{t-1}:=\cX\setminus h_{t-1},
      \qquad\text{so}\qquad
      \Pr\bs{\text{mistake at }t\mid h_{t-1}}=\Pr\bs{x_t\in E_{t-1}\mid  h_{t-1}}=\mu\br{E_{t-1}}.
\]

We now use results from the PAC-learning framework and derive bounds on the expected loss. We consider the learning algorithm $\mathcal{A}(S)=\clos_{\HC}\br{x_1,\dots,x_n}$ for an i.i.d.\ sample $S=\{(x_1,f^*(x_1)),\dots, (x_n,f^*(x_n))\}$. 
Define the error of a classifier as \[L_{\cD,f^*}(\mathcal{A}(S))=\Pr_{x\sim\cD}(\mathcal{A}(S)(x)\neq f^*(x))\] and note that for the closure algorithm this is $L_{\cD,f^*}(\mathcal{A}(S))=\mu(E_{n})$. We know from PAC-learning bounds on the sample complexity (as shown in \cite{darnstadt2015OptimalPACBound}) that for any \(0<\epsilon,\delta<1\)  the closure algorithm $\mathcal{A}$ taking a sample of size $n\geq \frac{c}{\varepsilon}(\vcdim+\log(\frac{1}{\delta}))$, where $c$ is some constant and $\vcdim$ the VC-dimension of $\HC$, has $\Pr_{S\sim\cD^n}(\mu(E_{n})\geq\varepsilon)\leq \delta$. \\
From this, we derive bounds for the expected loss of the learned classifier: in particular for $\delta_\varepsilon=\exp\br{-\varepsilon n/c+d}$ one has
\begin{equation*}
    \bE_{S_n\sim \mathcal{D}^n}\bs{\mu(E_{n})}=\int_{\varepsilon=0}^\infty \Pr\br{\mu(E_{n})>\varepsilon} \mathrm{d}\varepsilon \leq \int_{\varepsilon=0}^\infty \exp\br{-\varepsilon n/c+\vcdim}  \mathrm{d}\varepsilon\leq c^*\frac{\vcdim}{n}
\end{equation*}
for some constant $C^*>0$. \\

Transferring these results to the online setting, where at each time we have a sample of $t-1$ labelled data points, from which we output a classifier, we get the following
\begin{align}\label{eq:stoch-ub}
\begin{split}\bE_{x_1,\dots,x_T}&\bs{\cM\br{T}}=\bE_{x_1,\dots,x_T}\bs{\sum_{t=1}^T\mathbb{I}\bc{\text{mistake at time }t}}= \bE_{x_1,\dots,x_T}\bs{\sum_{t=1}^T\mathbb{I}\bc{E_{t-1}}(x_t)} \\
    &=\sum_{t=1}^T\bE_{x_1,\dots,x_t}\bs{\mathbb{I}\bc{E_{t-1}}(x_t)}=\sum_{t=1}^T\bE_{x_1,\dots,x_{t-1}}\bs{\mu(E_{t-1})} \leq \sum_{t=1}^T C\frac{\vcdim}{t-1}=\bigO{\vcdim\log T}.
    \end{split}
\end{align}

Combining~\cref{eq:det-ub,eq:stoch-ub} completes the proof.

\end{proof}

\begin{restatable}[General Lower Bound - Stochastic Adversary]{lem}{lbstochgeneral}\label{thm:general-lower-stoch}
Let $\HC\subseteq2^{\X}$ be a hypothesis class and $T\in\N^+$ be the time horizon.
There exists a distribution $\cD$ over $\X$ and $f^*\in\HC$ such that any deterministic learner $\cA'$  in the replay setting has expected error \[\bE\bs{\cM_T(\cA')}=\Om{\min\{\extdim{\HC},\log T \}}.\]
\end{restatable}

\begin{proof}
The proof combines the arguments of~\Cref{thm:general-adver} and~\Cref{thm:thresh-stoch}. 

Let \(f\) be the first hypothesis that \(\cA'\) yields. Since \(\cA'\) is deterministic, \(f\) is known to the adversary. Again, arguing as in~\Cref{rem:f-rep}, we assume w.l.o.g. that the learner starts with $\HC=\HC^f$ where $f\in\arg\min_{f \subseteq \cX} \tdim{\overline{\cH^f}}$ and $\tdim{\HCC}=\extdim{\HC}$. Therefore, in the following, let \(f \equiv 0\).

We consider here the case \(\log T > 2 \tdim{\overline{\cH}}\) and refer to~\Cref{thm:thresh-stoch} for the other case.
Pick \(x_1, \ldots, x_{\tau}\) and \(h_0, \ldots, h_{\tau}\), the witness set for \(\tdim{\overline{\cH}}\) and consider the order \(x_1 < x_2 < \ldots < x_\tau\). Define a distribution \(\cD\) on \(x_1, \ldots, x_\tau\) as in the lower bound proof of~\Cref{thm:thresh-stoch}: 
    $\Pr\br{x=x_k}=\frac{1}{2\cdot3^{k-1}}$ for $k=\{2,\dots,\tau\}$ and $\Pr\br{x=x_1}=1-\sum_{k=2}^\tau \Pr\br{x=x_k}$.

Let \(\hat x_1, \ldots, \hat x_T \overset{\iid}{\sim} \cD\) be the sampled points and define \(t_k \coloneqq \min \bc{t \in [T / 2], \text{ such that } \hat x_t \geq x_k}\) with \(t_k = \infty\) when \(\hat x_t < x_k\) for all \(t \in [T/2]\).
By construction we have \(1 = t_1 \leq t_2 \leq \ldots \leq t_{\tau} \leq \infty\).
Define \(A_k \coloneqq \bc{t_k < \infty \text{ and } \hat x_{t_k} = x_k}\). Arguing as in~\Cref{eq:tk_finite,eq:prob_ratio,eq:pr_ak} we get that \(\Pr(A_k) \geq 1/3\). 

To finish the proof, we consider two cases, as in~\Cref{thm:thresh-stoch}. Let \(\hat h_k\) be the hypothesis used by the learner the first time \(x_k\) was sampled from \(\cD\).

\noindent \textbf{First case.} Either, for each \(k \in [\tau]\), such that \(\hat x_{t_k} = x_k\), \(\hat h_k(x_k) = 0\). Then, when presented with positive label for \(x_k\), the learner makes a mistake, and we obtain \(\Omega(\tdim{\overline{\cH}})\) mistakes in total when \(f^\ast = h_{\tau}\).

\noindent \textbf{Second case.}
Otherwise, there exists \(k \in [\tau]\), such that \(\hat x_{t_k} = x_k\) and \(\hat h_k(x_k) = 1\). 
Here, we argue that the learner introduced a non-empty \(\Trp\) region similar to~\Cref{thm:general-adver}. This, by~\Cref{thm:no_trap}, implies that there exists an adversarial strategy to label future points, so that the learner will make \(\Omega(\tdim{\overline{H}})\) mistakes for this distribution.

\end{proof}

\stochgeneral*
\begin{proof}
    Combining \Cref{thm:general-upper-stoch} and \Cref{thm:general-lower-stoch} gives the claim. 
\end{proof}

\begin{restatable}[Convex Bodies - Stochastic Adversary Upper Bound]{lem}{stochconvexupper}
\label{lem:convex-stoch-upper}
Let \(d \in \mathbb{N}\), and define \(\HC_d = \{ C \subseteq \mathbb{R}^d : C \text{ is convex} \}\) as the class of all convex subsets of \(\mathbb{R}^d\). 
Let \(\mu\) be a probability distribution supported on a fixed convex set \(C^* \subseteq \mathcal{B}_1^d\), where \(\mathcal{B}_1^d\) denotes the unit ball in \(\mathbb{R}^d\).
Assume that \(\mu\) admits a density \(f\) which is upper bounded by some constant \(M<\infty\).  Then
there exists a learner $\cA$ (the
closure algorithm,~\Cref{alg:closure_explicit})
, such that for any time horizon \(T\ge 1\) the expected mistakes satisfy
\[
   \bE\bs{\cM_T\br{\cA}} =
   \begin{cases}
        \bigO{\log T} & \text{if }d=1,\\
        \bigO{T^{\frac{d-1}{d+1}}} & \text{if }d\ge 2.
     \end{cases}
\]
\end{restatable}
\begin{proof}
Consider $f^*=1$ and the convex-hull algorithm starting from $h_0=\emptyset$. 
After $t-1$ points the learner predicts with
\(h_{t-1}=\conv{x_1,\dots,x_{t-1}}\). A mistake at round $t$ occurs for the closure algorithm
iff $x_t\notin h_{t-1}$, \textit{i.e.} \(x_t\) is outside the convex
hull of all points seen so far, see \Cref{lem:clos_only_true}. We denote this error region as
\(E_{t-1}:=C^\star\setminus h_{t-1}\). The probability of error in round $t$, is thus given by
\begin{equation}\label{eq:missing-mass}
       \Pr_{x_t\sim\cD}\bs{\text{mistake at }t}=\Pr_{x_t\sim\cD}\bs{x_t\in E_{t-1}}=\mu\br{E_{t-1}}=\int_{E_{t-1}}f(x)dx.
\end{equation}
Results from \citep[Corollary 1]{Brunel2020Deviation} show\footnote{\(d_f(C^\ast,h_n)\) in~\citep{Brunel2020Deviation} is
precisely the integral in~\Cref{eq:missing-mass} for \(n=t-1\) and the
inequality follows by setting \(q=1\) and choosing the appropriate
constants.}
that there exists some constant \(C\) that depends only on
\(d, M\) such that,
\[
   \bE\bs{\mu\br{E_{t-1}}} \le C\br{t-1}^{-\nicefrac{2}{\br{d+1}}}.
\]
\noindent This holds for all \(t,d\ge 1\).\\
For \(d=1\), we have
\begin{align*}
    \bE_{x_1,\dots,x_{t}}\bs{\cM(T)} & = \bE_{x_1,\dots,x_{t}}\bs{\sum_{t=2}^T\mathbb{I}\bc{x_t\notin\conv{x_1,\dots,x_{t-1}}}+1} = \sum_{t=2}^T \bE_{x_1,\dots,x_{t-1}}\bs{\Pr_{x_t}\br{x_t\in E_{t-1}}}+1\\
    &=\sum_{t=2}^T \bE_{x_1,\dots,x_{t}}\bs{\mu(E_{t-1})}+1 \le C\sum_{t=2}^{T}\br{t-1}^{-1}+1 =\bigO{\log T}.
\end{align*}
   For \(d\ge 2\), using the same arguments we have
   \[\begin{aligned}
   \bE_{x_1,\dots,x_{t}}\bs{\cM(T)}&\le C\sum_{t=2}^{T}\br{t-1}^{-\nicefrac{2}{\br{d+1}}}+1\leq C\br{\int_{1}^{T-1}x^{-2/(d+1)}dx}+2\\
   &=C\br{\frac{d+1}{d-1}\br{(T-1)^{(d-1)/(d+1)}}}+2=O\br{T^{\frac{d-1}{d+1}}}.
   \end{aligned}
    \]
\end{proof}

\begin{restatable}[Convex Bodies - Stochastic Adversary Lower Bound]{lem}{stochconvexlower}
\label{lem:convex-stoch-lower}
    there exists a distribution $\cD$ and $f^*\in\HC$ such that any learner $\cA'$ in the replay setting has expected mistakes after round $T$
    \[
   \bE\bs{\cM\br{T}} =
   \begin{cases}
        \Om{\log T} & \text{if }d=1,\\
        \Om{T^{\frac{d-1}{d+1}}} & \text{if }d\ge 2.
     \end{cases}
\] 
\end{restatable}

\begin{proof}
For the lower bound, we consider $\cD$ to be the uniform distribution on a convex body $C^*$. \Cref{cor:hypothesis_chain} shows that the convex-hull algorithm (i.e. closure algorithm) is optimal as any other learner either incurs more error or creates a non-empty \gls{trap}. In particular, for a non-empty \gls{trap} the replay adversary can incur linear error in $T$ as the uniform distribution $\cD$ has constant measure $\mu$, see \Cref{thm:no_trap}.
Therefore, we can restrict ourselves to show that there exists a distribution $\cD$ (uniform distribution) and $f^*\in\HC$ such that the convex-hull algorithm, i.e. \emph{closure algorithm} incurs the claimed lower bounds. \\
This can be seen by using a lower bound argument for convex bodies from \cite[Theorem 5.2]{mammen1995AsymptoticalMinimaxRecovery}, to show that there exists constants $c$ which depend only on $d$ such that\footnote{For details see \cite[Corollary 4]{Brunel2020Deviation}} 
\begin{equation*}
    c (t-1)^{-\nicefrac{2}{d+1}}\leq \sup_{\text{ for convex body } C\subseteq C^*}\bE_{x_1,\dots,x_{t-1}\sim\cU(C)}\bs{\mu_C(E_{t-1})}
\end{equation*}

By the same arguments, one gets for $d=1$:
\begin{align*}
    \bE_{x_1,\dots,x_{t}}\bs{\cM(T)}=\sum_{t=2}^T \bE_{x_1,\dots,x_{t}}\bs{\mu(E_{t-1})}+1 \geq c\sum_{t=2}^{T}\br{t-1}^{-1} =\Om{\log T}.
\end{align*}
and for $d\geq 2$:
\begin{align*}
    \bE_{x_1,\dots,x_{t}}\bs{\cM(T)}=\sum_{t=2}^T \bE_{x_1,\dots,x_{t}}\bs{\mu(E_{t-1})}+1 \geq c\sum_{t=2}^{T}\br{t-1}^{-\nicefrac{2}{\br{d+1}}}+1\\
    \geq c\br{\int_{2}^{T}x^{-2/(d+1)}dx}=\Om{\log T}.
\end{align*}

\end{proof}

\begin{proof}[Proof of \Cref{thm:convex-stoch}]
    Combining \Cref{lem:convex-stoch-lower} and  \Cref{lem:convex-stoch-upper} gives the claim. 
\end{proof}

\subsection{Intersection-Closed Classes}
\label{app:proper-int-closed}

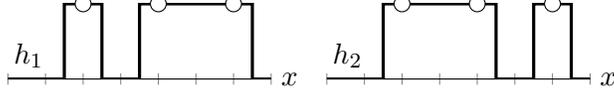
\begin{figure}
    \centering
    \begin{tikzpicture}
  \begin{axis}[
    x=0.5cm,y=1cm,
    axis lines=middle, %
    axis y line=none,
    axis line style={-}, %
    xlabel={$x$}, %
    ylabel={$y$},
    xlabel style={at={(ticklabel* cs:1.0)}, anchor= west}, %
    ymin=0, %
    ymax=1, %
    xmin=0, %
    xmax=7, %
    xtick= {0,1,...,7},
    xticklabels={},
  ]
    \addplot+[const plot, no marks, draw=black, very thick] coordinates {(0,0)(1.5,0) (1.5,1) (2.5,1) (2.5,0) (3.5,0) (3.5,1) (6.5,1) (6.5,0) (7,0)} node[above, pos=0.05,black]{$h_1$};
    \addplot [only marks,mark=*,mark size=3pt, mark options={fill=white}] table {
    2 1
    4 1
    6 1
};
  \end{axis}
  \end{tikzpicture}
  \begin{tikzpicture}
  \begin{axis}[
    x=0.5cm,y=1cm,
    axis lines=middle, %
    axis y line=none,
    axis line style={-}, %
    xlabel style={at={(ticklabel* cs:1.0)}, anchor= west},
    xlabel={$x$}, %
    ylabel={$y$},
    ymin=0, %
    ymax=1, %
    xmin=0, %
    xmax=7, %
    xtick= {0,1,...,7},
    xticklabels={},
  ]
    \addplot+[const plot, no marks, draw=black, very thick] coordinates {(0,0) (1.5,0) (1.5,1) (4.5,1) (4.5,0) (5.5,0) (5.5,1) (6.5,1) (6.5,0) (7,0)} node[above,pos=0.05,black]{$h_2$};

    \addplot [only marks,mark=*,mark size=3pt, mark options={fill=white}] table {
    2 1
    4 1
    6 1
};
  \end{axis}
  \end{tikzpicture}
    \caption{\Cref{ex:non-int} of non intersection-closed classes and two possible functions $h_1$ and $h_2$ the learner can output to be consistent after having observed 3 samples with label 1.}
    \label{fig:example_non-int}
\end{figure}

\intersectionadv*
\begin{proof}
Proving the direction that \(\HC\) is learnable if there exists an \(f\)-representation of \(\HC\) such that \(\HC^f\) is intersection closed is an immediate consequence of the following observation. Choosing that \(f\)-representation and running~\Cref{alg:closure_explicit} yields a closure algorithm on an intersection closed class and thus \(h_t\) always remains within \(\HC^f\). Then the learner always outputs \(h_t\oplus f\in \HC\).

The proof proceeds in two main steps. First, we reduce the problem by
showing that the learnability of a class $\HC$ from an arbitrary
starting hypothesis $h_1$ is equivalent to the learnability of a
corresponding ``flipped'' class $\HC^f$ when starting from $h_{\text{min}}$. Second, we
prove by contradiction that any class that is
learnable starting from $h_{\text{min}}\in\HC$ must be intersection-closed.

Let $\cA$ be a proper learning algorithm that guarantees a finite
mistake bound for $\HC$, starting from a hypothesis $h_1 \in \HC$. Note that by definition of $h_{\min}=\cap_{h\in\HC} h\subseteq h_1$.  
Set $f:=h_1\setminus h_{\text{min}}$ and define the $f$-representation of $\HC$ as $\HC^f = \{h^f \mid h \in \HC\}$, where now $(h_1)^f=h_{\min}$ by definition. This mapping is a bijection, and there exists a corresponding algorithm $\cA$ for $\HC^f$ whose starting
hypothesis is $(h_1)^f  = h_{\min}\in\HC$. $\cA'$ also flips the received samples $(x_i,y_i^f)$ where $y_i^f=1-y_i$ when $f(x_i)=1$ and $y_i^f=y_i$ else. Thus w.l.o.g., we can restrict ourselves to the setting where the learner starts with $h_{\min}$ by considering the corresponding $f$-representation of the class and a learner $\cA'$ that flips the labels and predicts with hypothesis in $\HC^f$.\\
Thus, we will now prove that any class $\HC$, which is learnable starting
from $h_{\min}\in\HC$, must be intersection-closed. By contradiction, assume there exist, $h_1,h_2\in\HC$ such that $h_1'\cap h_2'=k'\notin\HC$.
The adaptive adversary chooses a target which allows playing samples $S \subseteq k'$ with label 1 s.t. $S$ is chosen as a minimal subset of $k'$ such that
 the closure is not a hypothesis from $\HC$: \[S\in\arg\min_{S'\subseteq k'}\{\clos_\HC(S')\notin\HC\}.\] 
Up to this point $t':=|S|$ the learner makes at most $t'=|S|\leq|k'|$ mistakes. \\
As the learner started from $h_{\min}\in\HC$ the reliable version space of consistent hypothesis with $S$ is $\VS_{t'}=\{h\in\HC \mid \clos_\HC(S')\subsetneq h\}$. 
A proper learner that updates the hypothesis from this point on, will have to choose some $h_{o}\in \VS_{t'}$ (choosing some $h_{o}\notin\VS_{t'}$ is a worse strategy and can be used by an adaptive adversary to incur errors in every round, see \Cref{cor:hypothesis_chain}).  
However, by assumption of non-intersection closed $\HC$, there exists a point $z\in h_{o}\setminus \clos_\HC(S')$ as $\clos_\HC(S')\notin \HC$ and $h_{o}\in\HC$. 
Therefore, we have replay hypothesis $h_{o}(z)=1$ and $h_{\min}(z)=0$. If this would not hold, i.e. $h_{\min}(z)=1$ then $\forall h\in\HC$ one had $h(z)=1$ which contradicts $z\notin \clos_{\HC}(S)$ by definition of closure. Additionally, \[z\notin \clos_\HC(S')\subseteq k'=h_1'\cap h_2'\] implies that there exists an $i\in\{1,2\}$ such that $h_i'(z)=0$ and both $h_{o}$ and $h_i'$ are in $\VS_{t'}$. This satisfies the conditions for a non-empty \gls{trap} as defined in \Cref{defn:trapt}, as $z\in\Trp_{t'}$. 
Thus, by \Cref{thm:no_trap}, the adaptive adversary can force the learner to make $\Omega(T)$   mistakes (note that potentially $\Omega(T-|k'|)$, depending on how the learner predicts on the first samples, but $|k'|$ is considered constant).

\end{proof}

\end{document}